\documentclass{tlp}

\usepackage{lineno}

\usepackage{graphicx} 
\usepackage{microtype}
\usepackage{amsmath}
\usepackage{amssymb}
\usepackage{graphicx}
  \usepackage{multirow}
\usepackage{hyperref}
\usepackage{fancyvrb}
\usepackage{booktabs}
\usepackage{float} 
\usepackage{caption}
\usepackage{cases}
\usepackage{bm}
\usepackage{rotating}
\usepackage{subcaption}
\usepackage{lmodern}

\DeclareCaptionType{listing}
\floatname{listing}{Listing}
\newcommand{\naf}{\neg}
\newcommand{\supp}[2]{\mathit{sp}_{#1}^{#2}}
\newcommand{\bd}[2]{\mathit{bd}_{#1}^{#2}}
\newcommand{\SCC}{\mathit{SCC}}
\newcommand{\TrRk}{\mathit{TrRk}}
\newcommand{\TrHd}{\mathit{TrHd}}
\newcommand{\TrBd}{\mathit{TrBd}}
\newcommand{\TrRule}{\mathit{TrRule}}
\newcommand{\Tr}{\mathit{Tr}}

\newcommand{\nop}[1]{}

\newcommand{\mysubsubsection}[1]{\smallskip\noindent{\bf#1}~}

\newcommand{\myparagraph}[1]{\vspace{2pt}\noindent\emph{#1}~}

\newcounter{myexample}
\newenvironment{myexample}[1][]{\refstepcounter{myexample}\par\medskip
   \noindent \mbox{\emph{Example~\themyexample. #1}}\quad \rmfamily}{\smallskip}

\newcommand{\varsP}{\ensuremath{\mathcal{V}_P}}
\newcommand{\atomsP}{\ensuremath{\mathcal{A}_P}}
\newcommand{\linAtomsP}{\ensuremath{\atomsP^{\mathit{lin}}}}
\newcommand{\costIP}{\ensuremath{c_P(I)}}
\newcommand{\costcalIP}{\ensuremath{c_P(\mathcal{I})}}
\newcommand{\AS}{\ensuremath{\mathit{AS}}}
\newcommand{\solver}{\text{asp-fzn}}

\newtheorem{theorem}{Theorem}
\newtheorem{proposition}{Proposition}
\newtheorem{definition}{Definition}
\newtheorem{lemma}{Lemma}
\newtheorem{corollary}{Corollary}

\newcommand{\citeauthorS}[1]{\citeauthor{#1}'s \citeyearpar{#1}}

\newcommand{\mysummanddisp}[2]{\hspace*{#1} #2}
\newcommand{\mysummanddispq}[2]{\mysummanddisp{#1}{#2} \quad }

\newcommand{\dataURL}{{\small\url{https://doi.org/10.5281/zenodo.16267414}}}

\newtheorem{innercustomgeneric}{\customgenericname}
\providecommand{\customgenericname}{}
\newcommand{\newcustomtheorem}[2]{%
  \newenvironment{#1}[1]
  {%
   \renewcommand\customgenericname{#2}%
   \renewcommand\theinnercustomgeneric{##1}%
   \innercustomgeneric
  }
  {\endinnercustomgeneric}
}

\newcommand{\new}[1]{\textcolor{black}{#1}}

\newcustomtheorem{customthm}{Theorem}
\newcustomtheorem{customlemma}{Lemma}
\newcustomtheorem{customprop}{Proposition}

\allowdisplaybreaks[1]

\newenvironment{compactitemize}
{\begin{list}{$\bullet$}{%
\setlength{\topsep}{2pt} 
\setlength{\leftmargin}{0pt} 
\setlength{\itemindent}{10pt}}} 
{\end{list}}

\newcommand{\myalign}[3]{

\vspace*{#1}

\begin{align}
#2
\end{align}

\vspace*{#3}

\noindent
}

\begin{document}


\lefttitle{ASP-FZN}

\jnlPage{1}{17}
\jnlDoiYr{2025}
\doival{10.1017/xxxxx}

\title[ASP-FZN]{{ASP-FZN}: A Translation-based \\ Constraint Answer Set Solver}


\begin{authgrp}
\author{ \sn{Thomas} \gn{Eiter}$^1$, \sn{Tobias} \gn{Geibinger}$^1$, \sn{Tobias} \gn{Kaminski}$^3$, \\ \sn{Nysret} \gn{Musliu}$^1$, \gn{Johannes} \sn{Oetsch}$^2$}
\affiliation{%
$^1$ TU Wien, Austria \\
$^2$ Jönköping University, Sweden \\
$^3$ Bosch Center for AI, Germany \\
\{firstname.lastname\}@tuwien.ac.at, johannes.oetsch@ju.se, tobias.kaminski@de.bosch.com}
\author{
}      
\end{authgrp}


\maketitle

\begin{abstract}
We present the solver asp-fzn for Constraint Answer Set Programming (CASP), which extends ASP with linear constraints. 
Our approach is based on translating CASP programs into the solver-independent FlatZinc language that supports several Constraint Programming and Integer Programming backend solvers. 
Our solver supports a rich language of linear constraints, including some common global constraints. As for evaluation, we show that asp-fzn is competitive with state-of-the-art ASP solvers on benchmarks taken from past ASP competitions. 
Furthermore, we evaluate it on several CASP problems from the literature and compare its performance with clingcon, which is a prominent CASP solver that supports most of the asp-fzn language. The performance of asp-fzn is very promising as it is already competitive on plain ASP and even outperforms clingcon on some CASP benchmarks. 
\end{abstract}



\section{Introduction}

Answer Set Programming (ASP) is a 
popular rule-based formalism 
for various AI applications and combinatorial problem-solving, where 
a problem is represented by an ASP program
whose answer sets (models) represent the solutions, 
potentially also under certain optimization criteria.
Especially for modeling industrial problems, Constraint Answer Set Programming (CASP), which adds reasoning over linear constraints to ASP, proved to be quite effective, e.g., for scheduling problems~\citep{DBLP:conf/lpnmr/Balduccini11,geibinger2021constraint}.

While efficient CASP solvers are available, cf.\ 
the recent survey by 
\cite{DBLP:journals/tplp/Lierler23}, they still often lag behind state-of-the-art Constraint Programming (CP) or Mixed Integer Programming (MIP) solvers for certain problem domains.
%
CASP solvers are either based on dedicated algorithms or translations into related formalisms such as Satisfiability Modulo Theory (SMT). The latter approach is inspired by similar works for solving plain ASP programs, but has the downside that SMT solvers generally lack optimization features and are thus not applicable for many problems appearing in practice.
This begs the question why, instead of targeting SMT, the translation is not aimed at FlatZinc~\citep{DBLP:conf/cp/NethercoteSBBDT07}, which is a solver-independent intermediate language that offers those lacking optimization features and works with many modern CP and MIP solvers as backend engines. The lack of such an approach was also noted by \cite{DBLP:journals/tplp/Lierler23}.

To fill this gap, we present the CASP solver \solver{}, which translates CASP programs into FlatZinc, thereby leveraging decades of CP and MIP solver engineering for efficient (optimal) solution finding.
To support modern CASP encodings featuring not only linear constraints but also specific scheduling constraints and ASP constructs like variables, aggregates, choice, and disjunction, we uilize gringo's theory interface 
\citep{DBLP:journals/tplp/GebserKKS19,DBLP:journals/tplp/KaminskiRSW23} 
to obtain a simplified program format.
Our approach then combines and extends ideas from translation-based ASP solving~\citep{DBLP:conf/ijcai/AlvianoD16,DBLP:journals/corr/abs-2308-15888} to create a FlatZinc representation encompassing all mentioned constructs.
By the richness of FlatZinc, incorporating complex global constraints and hybrid optimization of both ASP weak constraints and objectives over linear variables is easy.
Notably, those features are not yet fully supported by other state-of-the-art CASP solvers like clingcon~\citep{DBLP:journals/tplp/BanbaraKOS17,DBLP:journals/algorithms/CabalarFSW23}.

Our main contributions are briefly summarized as follows:
\begin{compactitemize}
    \item We present a translation $\Tr(P)$ of head-cycle-free CASP programs $P$ into a low level constraint language, which can be parsed by several state-of-the-art CP and MIP solvers.
    \item Our translation extends and combines existing concepts from the literature and supports not only linear constraints but also choice rules, weight rules, disjunction, and optimization.
    \item We show that 
    $\Tr(P)$ captures all answer sets of $P$, with a one-to-one or many-to-one mapping to its models, depending on the presence of correspondence constraints.
    \item We introduce our solver \solver{}, which implements the described translation and utilizes external grounding and a parametric backend solver for answer-set optimization.
    \item We evaluate \solver{} using different backend solvers against state-of-the-art (C)ASP solvers, finding that it is competitive on plain ASP and outperforms clingcon on some CASP benchmarks.
\end{compactitemize}

The solver \solver{} thus enables solving expressive (C)ASP programs via CP and MIP solvers, leveraging their strengths. As with SAT-based ASP solvers, this approach benefits from the substantial engineering behind these solvers, future advancements, and the decoupling of (C)ASP solving from specialized, maintenance-heavy algorithms.

\section{Preliminaries}


We consider propositional Answer Set Programming (ASP)
\citep{brew-etal-11-asp}
with \emph{programs} $P$ that are sets of rules $r$ of the form
\myalign{-3ex}
{\label{eqn:rule}
    H \leftarrow B
}{-0.5ex}
where $H$ is the \emph{head} of the rule and $B$
its \emph{body}, also denoted by  $H(r)$  and $B(r)$, respectively;
by $\atomsP$ we denote the set of all propositional atoms occurring in $P$.
We distinguish two types of rules: 1) \emph{disjunctive rules}
and 2) \emph{choice rules}, where $H$ has the form 
\linebreak
\begin{minipage}{0.45\textwidth}
\myalign{-2.5ex}
{\label{eqn:disj}
    a_1 \mid \dots \mid a_m & & \text{({\em disjunctive head})}
}{-4.25ex}
\end{minipage}
\begin{minipage}{0.55\textwidth}
\myalign{-2.5ex}
{\label{eqn:choice}
\text{respectively}~~     \{ a_1 , \dots , a_m \} & & \text{ ({\em choice
head})}
}{-4.25ex}
\end{minipage}
\vspace{0pt}

\noindent 
where all $a_i$ are 
atoms. Intuitively, "$\mid$" stands for logical disjunction, i.e., at least one of the atoms must hold,  
while for choice, any number of $a_i$ can be true if
$H$
is true.
A disjunctive rule is a \emph{constraint rule} if $H(r)=\emptyset$  and a \emph{normal rule} if $|H(r)|=1$.



Furthermore, we consider two types of rule bodies: 1) \emph{normal} rule bodies of the form
\begin{align}\label{eqn:normal}
    b_1 , \dots , b_k, \naf b_{k+1}, \dots , \naf b_{n}
\end{align}
where all $b_i$ are 
atoms, $\naf$ is negation as failure, and 
 ``,'' is 
 conjunction, and 2)
%
\emph{weighted} rule bodies 
\begin{align}\label{eqn:weighted}
    l \leq \{ b_1 : w_1, \dots , b_k : w_k, \naf b_{k+1} : w_{k+1}, \dots , \naf b_n : w_n \}
\end{align}
where all $b_i$ are atoms, all  $w_i$ are integer \emph{weights}, and $l$ is the integer \emph{lower bound}; we let $B^+(r)= \{b_1,\ldots,b_k\}$ and $B^-(r)\,{=}\,\{b_{k+1},\ldots, b_n\}$.

By slight abuse of notation, $a\in H(r)$ denotes that atom $a$ occurs in 
$H(r)$
and $l\in B(r)$ that literal $l$, i.e., an atom or its negation, occurs in $B(r)$.
We further let $w_b^r$  denote the weight of atom $b$ in the body of rule $r$, let $\top$ denote an empty conjunction, 
and let $\bot$ denote an empty rule head.


\begin{myexample}\label{ex:running}
    Consider the program $P_1 = \{ \ \{a,b\} \leftarrow c,\  \bot \leftarrow 3\leq \{ a : 1, b : 2 \} ,\ c \leftarrow \naf d \ \}$.
    The first rule of $P_1$ is a choice rule with normal body, the second rule is a constraint rule with a weighted body, and the last rule is a normal rule.
\end{myexample}
\vspace*{-0.25\baselineskip}

\myparagraph{Semantics.}
An \emph{interpretation} of a program $P$  is a set $I \subseteq \atomsP$ of  atoms,
which satisfies a disjunctive head 
(\ref{eqn:disj}) if $a_i \in I$ for some $i \in [1,m]$, and satisfies every choice rule head (\ref{eqn:choice}). 

Given a rule $r$ and an interpretation $I$, $I\models H(r)$ denotes that $I$ satisfies the head of $r$. Satisfaction of the body $B(r)$ by $I$, denoted $I \models B(r)$, is as follows: 1) 
for a normal rule body (\ref{eqn:normal}),  $b_i \in I$ for every $i\in[1,k]$ and $b_j \not\in I$ for every $j\in(k,n]$ must hold; 2)
for a weighted rule body (\ref{eqn:weighted}), the  following linear inequality must hold:
 $l \;\leq \sum_{i\in[1,k], b_i\in I} w_i + \sum_{j \in(j,n], b_j\not\in I} w_j.$

An interpretation $I$ satisfies a rule $r$, denoted $I \models r$, whenever $I\models B(r)$ implies $I\models H(r)$ and $I$ is a model of program $P$, denoted $I\models P$, if $I \models r$ for all $r \in P$.

\myparagraph{Answer sets.} The \emph{(FLP) reduct} $P^I$ of program $P$ w.r.t.\ interpretation $I$ is the program containing, for
each $r \in P$ s.t.\ $I\models B(r)$, the following rules:
  (1)  if 
    $r$ is disjunctive, $H(r) \leftarrow B(r)$, and
 (2) if $r$ is a choice rule, for each $a \in H(r)$  the rule $a \leftarrow B^+(r)$  
if $B(r)$ is normal and $
a \leftarrow l' \leq \{ b_1 : w_1, \dots , b_k : w_k \}$ if $B(r)$ is a weighted body (\ref{eqn:normal}), where  
 $l'= \mathit{max}(0, l - \sum_{j \in (k,n], b_j\not\in I} w_j)$. 


Finally, an interpretation $I$ is an \emph{answer set} of program $P$ if $I$ is a $\subseteq$-minimal model of $P^I$. The set of all answer-sets of $P$ is denoted by $\AS(P)$.

\begin{myexample}\label{ex:running_ctd1}
Program $P_1$ from Example~\ref{ex:running} has $\AS(P_1)=\{ \{c\}, \{c,a\}, \{c,b\} \}$.
\end{myexample}

We allow programs $P$ to contain also a single \emph{minimization} statement (Priority levels can be added and compiled to this form using known techniques):
\begin{align}\label{eqn:min}
    \mathit{min} \ a_1 : w_1, \dots , a_k : w_k, \naf a_{k+1} : w_{k+1}, \dots , \naf a_n : w_n
\end{align}
The \emph{cost} of interpretation $I$ is $\costIP = \sum_{i \in [1,k], a_i\in I} w_i + \sum_{j \in (k,n], a_j\not\in I} w_j$
and 0 if $P$ has no minimization.
An answer set $I$ of $P$ is \emph{optimal} if $\costIP$ is minimal over 
$\AS(P)$.

\subsection{Constraint Answer Set Programming}

We next introduce \emph{linear constraints and variables} in our programs, thus turning to \emph{Constraint Answer Set Programming (CASP)}.
We consider a countable set $\mathcal{V}$  of linear variables. Each $v \in \mathcal{V}$ has a \emph{domain} $D(v)$ that is assumed to be an integer range, which defaults
to $[-\infty, +\infty]$;
it can be restricted by a \emph{domain constraint} of the form
\begin{align}\label{eqn:dom}
    v \in [l,u]
\end{align}
where $l$ and $u$, $l \leq u$,  are integer lower and upper bounds. 
In general, bounding the linear variables is not required but the CASP solver 
might infer bounds or fallback to some default values.

A \emph{linear constraint} is of the form
\begin{align}\label{eqn:lin}
    a \leftrightarrow v_1 \cdot w_1 + \dots + v_n \cdot w_n \circ g
\end{align}
where $a$ is an
atom, each $v_i$ is a linear variable, each $w_i$ and $g$ are integer constants, and $\circ \in \{ <,>,=,\neq, \leq, \geq \}$ is a comparison operator.
Intuitively, 
$a$ is constrained to the truth value of the linear constraint. Syntactically, $a$ can appear in the bodies of standard ASP rules (\ref{eqn:rule}). 
For any CASP program $P$, we denote by $\varsP$ and $\linAtomsP$ the sets of all linear variables and all propositional atoms occurring in linear constraints of $P$, respectively.

We additionally allow a CASP to contain global constraints. An \emph{alldifferent} constraint is of the form 
\myalign{-5ex}
{\label{eqn:alldiff}
     \&\mathit{distinct}\{ v_1 , \dots , v_n\}
}{-0.5ex}
where each $v_i$ is a linear variable and all are constrained to be pair-wise different.
A \emph{cumulative} \/
constraint is of the form 
\myalign{-4ex}
{\label{eqn:cumulative}
     \&\mathit{cumulative}\{ (s_1,l_1,r_1) , \dots , (s_n,l_n,r_n)\} \leq g
}{-0.5ex}
where $s_i$ is a linear variable representing the start of each interval, $l_i$ is a linear variable representing the length, $r_i$ is a linear variable denoting the resource usage, and $g$ is an integer bound. The constraint then enforces that at each time point, the sum of the resource usages of the overlapping intervals does not exceed $g$.
A global \emph{disjoint}\/ constraint is of form 
$\&\mathit{disjoint}\{ (s_1,l_1) , \dots , (s_n,l_n)\}$
and can be seen as a special case of a constraint (\ref{eqn:cumulative})  where $r_i$ and $g$ are assumed to be $1$. 

\myparagraph{Semantics.} An \emph{extended (e-)}\/ interpretation for a CASP program $P$ is a tuple $\mathcal{I}=\langle I,\delta \rangle$ where $I$ is a set of propositional atoms and $\delta: \varsP \rightarrow \mathbb{Z}$
is an \emph{assignment} of integers to linear variables $\varsP$.
Satisfaction $\mathcal{I} \models \phi$, where $\phi$ is a head, body, rule, program etc., is defined as above via $I$.

An e-interpretation $\mathcal{I}=\langle I,\delta \rangle$ is a \emph{constraint answer set} of $P$ if  (1) $I$ is an answer set of $P \cup \{ \{a\} \leftarrow \  \mid a \in \linAtomsP \}$, (2) for each domain constraint (\ref{eqn:dom}) in $P$, $\delta(v)\in [l,u]$, and (3) for each linear constraint (\ref{eqn:lin}) in $P$, $a \in I$ iff $\sum_{1\leq i \leq n} \delta(v_i) \cdot w_i \circ g$.
By slight abuse of notation we also use $\AS(P)$ to refer to the constraint answer sets of a CASP program $P$.

\begin{myexample}[ (Ex.~\ref{ex:running}~cont'd) ]\label{ex:running_ctd2}
Let $P_2 \,{=}\, P_1 \,{\cup}\, \{ x \,{\in}\, [0,2], \ y \,{\in}\, [0,1], \ d \leftrightarrow x \,{\cdot}\, 1 \,{+}\, y \,{\cdot}\, 1 \,{\neq}\, 3 \}$.
    Clearly, $P_2$ is a CASP program with $\AS(P_2) = \{\; \langle \{ c \}, \{ (x,2),(y,1) \} \rangle$, 
    $\langle \{ b, c \}, \{ (x,2), (y,1) \} \rangle$, 
    $\langle \{ a, c \}, \{ (x,2), (y,1) \} \rangle$, 
    $\langle \{ d \}, \{ (x,0), (y,0) \} \rangle$,  $\langle \{ d \}, \{ (x,1), (y,0) \} \rangle$,
        $\langle \{ d \}, \{ (x,2), (y,0) \} \rangle$, \,
    $\langle \{ d \}, \{ (x,1), (y,1) \} \rangle$,\,
    $\langle \{ d \}, \{ (x,0), (y,1) \} \rangle
    \; \}$.
\end{myexample}

For CASP programs, we allow minimization over the linear variables with statements 

\myalign{-3ex}
{\label{eqn:lin_min}
     \mathit{min} \ v_1 \cdot w_1 + \dots + v_n \cdot w_n 
}{-0.5ex}
where each $v_i$ is a linear variable and each $w_i$ is an integer constant.
The cost $\costcalIP$ of an e-interpretation $\mathcal{I}$ of a CASP program $P$ is 
the sum of the costs determined by statements (\ref{eqn:min}) and (\ref{eqn:lin_min}), and optimal constraint answer sets are, {\it mutatis mutandis}, analogous to optimal answer sets.

\vspace*{-0.5\baselineskip}

\section{Supported Models and Ranked Interpretations}

Prior to the translation, we introduce a few auxiliary concepts.
The \emph{positive dependency graph} of a (C)ASP program $P$
is $\mathit{DG}^+_P=(V,E)$ 
with nodes $V=\atomsP$ and
 edges $(a,b) \in E$
for all atoms $a,b$  s.t.\  $a \in H(r)$ and $b \in B^+(r)$ for some  rule $r\in P$. 
A program $P$ is \emph{tight} if $\mathit{DG}^+_P$ is acyclic; a rule $r \in P$ is \emph{locally tight} if $H(r)\cap B^+(r)=\emptyset$.
%
We denote for $a\in \atomsP$ by  $\SCC_P(a)$ its strongly connected component (SCC) in $\mathit{DG}^+_P$, which is 
\emph{non-trivial} if $|\SCC_P(a)| > 1$. A program
$P$ is \emph{head-cycle free (HCF)}\/ if every rule $r\in P$ and distinct $a\neq b \in H(r)$ fulfill $b \notin \SCC_P(a)$.

Clearly, a tight program has no non-trivial SCCs and are HCF, while a non-tight program may or may not be HCF. 
In the sequel, we assume that all programs are HCF; 
while this excludes some programs, it still allows us with minimization to embrace the class of $\mathsf{NP}$-optimization problems\footnote{see \url{https://complexityzoo.net/Complexity_Zoo}} as follows from \citep{DBLP:journals/amai/EiterFFW07},  and thus most problems appearing in practice.

Recall that for an ASP program $P$, an interpretation $I$ is a \emph{supported model} of $P$ if (1) $I\models P$ and (2) for each $a \in I$ some rule $r \in P$ exists such that $I\models B(r)$, $a \in H(r)$, and $H(r)\cap I = \{a\}$ if $r$ is disjunctive.
%
For tight ASP programs, supported models and answer sets coincide~\citep{DBLP:journals/tplp/ErdemL03}.
For non-tight HCF programs, we consider \emph{ranked supported models} 
as follows. 

We assume that $\varsP$ includes for each atom
$a \in \atomsP$ a variable $\ell_a$ not occurring in $P$; intuitively, it denotes the \emph{rank} (or \emph{level}) of $a$. 
An e-interpretation $\mathcal{I}=\langle I, \delta \rangle$ is \emph{ranked}, if for each $a \,{\in}\, \atomsP$, $\delta(\ell_a)  \,{=}\,\infty$ if $a\,\not\in I$ and $\delta(\ell_a)<\infty$ otherwise.
%
A rule $r$ 
\emph{supports} atom $a \,{\in}\, I$, if $a \,{\in}\, H(r)$, $H(r) \,{\cap}\, I = \{ a \}$ if $r$ is disjunctive, and $B(r)$ fulfills: 1) if $B(r)$ is 
normal (form (\ref{eqn:normal})),  
(i) $\delta(\ell_{b_i}) < \delta(\ell_a)$ for each 
$i \in [1,k]$ and (ii) $b_j \not\in I$ for each 
$j \in (k,n]$
and 2) if $B(r)$ is a weighted rule body, 
\begin{equation}
\label{eq:r-support}
    l\ \leq \sum_{b\in B^+(r), \delta(\ell_{b}) < \delta(\ell_a)}  \mysummanddispq{-0.75cm}{w_b^r} + \sum_{b\in B^-(r), b\not\in I} \mysummanddispq{-0.5cm}{w_b^r}\,.
\end{equation}
\begin{definition}
A \emph{ranked supported model} of program $P$ is a ranked interpretation $\mathcal{I}=\langle I, \delta \rangle$ of $P$ such that $\mathcal{I} \models P$ and each $a \in I$ is supported by some rule $r \in P$.
\end{definition}
We then obtain:
\begin{proposition}\label{prop:as_equiv_rank_supp}
For every HCF program $P$, $I\in \AS(P)$ iff $\langle I, \delta \rangle$ is a ranked supported model of $P$ for some 
$\delta$.
\end{proposition}

We can refine this characterization by considering the modular structure of answer sets along the SCCs. 
A ranked interpretation $\langle I, \delta \rangle$ of program $P$ is \emph{modular},  if  each $a \,{\in}\, I$ fulfills $\delta(\ell_a) \leq |\SCC_P(a)|$; hence true atoms in trivial components must have rank~1.
We say a rule $r$ 
{\em scc-supports}\/ $a \in I$
by changing  in "$r$ supports $a$" above for $B(r)$ condition (i) in case 1) to "$b_i\in I$ for each 
$i \in [1,k]$ where $\delta(\ell_{b_i}) < \delta(\ell_a)$ if $b_i\in \SCC_P(a)$", and  condition (\ref{eq:r-support}) in case 2)  to
$$l\ \leq \sum_{b \in B^+(r)\setminus\SCC_P(a)} \mysummanddispq{-0.75cm}{w_b^r} + \sum_{b \in B^+(r)\cap\SCC_P(a),  \delta(\ell_{b_i}) < \delta(\ell_a)} \mysummanddispq{-0.95cm}{w_b^r} + \sum_{b\in B^-(r)\setminus I} \mysummanddisp{-0.25cm}{w_b^r}\,,
$$
and define scc-supported models analogous to supported models. We then can show:
\begin{proposition}\label{prop:as_equiv_mod_rank_supp}
For every HCF program $P$, $I\in \AS(P)$ iff $\langle I, \delta \rangle$ is a modular ranked scc-supported model of $P$ for some level assignment $\delta$.
\end{proposition}

\vspace*{-\baselineskip}
\section{Translation}

In this section, we describe our translation of a (C)ASP program $P$ into a constraint program.
We assume that the considered program adheres to the following property.
\begin{definition}\label{def:partially_shifted}
    A HCF program $P$ is called \emph{partially shifted} if every rule $r\in P$ with a weighted body $B(r)$ fulfills either $|H(r)|\leq 1$ or $H(r)\cap \SCC_P(a) =\emptyset$ for every $a \in B^+(r)$.
\end{definition}

The property is named so because any HCF program can be transformed into partially shifted form by applying the well-known \emph{shifting} operation~\citep{DBLP:journals/amai/Ben-EliyahuD94} to the violating rules, resulting in two rules that satisfy the property.


For CASP programs, the translation simply includes the theory atoms as reified constraints and the domain constraints are used as bounds of the introduced variables. 
If there are no bounds, we simply declare the variables as integer and delegate the handling of unbounded variables to the underlying FlatZinc solver. 
Minimization statements must be combined into a single objective, which is trivial in absence of priority levels. For 
priority level minimization, we rely on well-known methods to compile them away. 

\subsection{Translation Constraints}

The translation, $\Tr(P)$ consists of serveral groups of constraints, which encode different aspects of an answer set of a (C)ASP program $P$:
\begin{compactitemize}
    \item ranking constraints $\TrRk(P)$, which encode the level ranking constraint;
    \item rule body constraints $\mathit{TrBd(r)}$, which encode the satisfaction of rules bodies;
    \item rule head constraints $\mathit{TrHd(r)}$,  which must be satisfied when rule bodies fire; and 
    \item supportedness constraints $\mathit{TrSupp(P)}$, ensuring that true atoms are supported. 
\end{compactitemize}
The complete translation for a program $\Tr(P)$ is then given by
$$\textstyle\Tr(P) = \TrRk(P)  \cup \bigcup_{r\in P} \mathit{TrRule(r)} \cup \mathit{TrSupp(P)}\,,$$
where $\mathit{TrRule(r)} = \mathit{TrBd(r)} \cup \mathit{TrHd(r)}$ is the combined body and head translation of  $r$.

\mysubsubsection{Ranking Constraints $\bm{\mathit{TrRk(P)}}$.}
First, we introduce some auxiliary atoms to handle the level ranking constraints, which follows the formulation given by \cite{DBLP:journals/corr/abs-2308-15888}.
Note that we assume that there are no tautological rules, i.e., $\mathit{DG}^+_P$ has no self-loops.

For each atom $a$ such that $|\SCC_P(a)|>1$, we introduce an integer variable $\ell_a$ with domain $[1,|\SCC_P(a)|+1]$
and add the following reified constraint to the translation:
\begin{align}
    \ell_a \leq |\SCC_P(a)| \;\leftrightarrow a\,. \label{eqn:rank_pos_level}
\end{align}
The constraint enforces that atom $a$ has rank $|\SCC_P(a)|+1$ iff $a$ is set to false.
Now, for all $b\in \SCC_P(a)$ such that $\mathit{DG}^+_P$ has an edge  $(a,b)$, we add a boolean auxiliary variable $\mathit{dep}_{a,b}$ and 
\begin{align}
     \ell_a - \ell_b \geq 1\ \leftrightarrow \mathit{dep}_{a,b} \label{eqn:dep_var}
\end{align}
which ensures that $\mathit{dep}_{a,b}$ is true iff $a$ has higher rank than $b$. 
The rank defined by these constraints is not \emph{strict}, i.e., an answer set may have multiple rankings. 
To enforce strictness, we add 

\begin{minipage}{0.45\textwidth}
\begin{align}
     \ell_a - \ell_b \geq 2\ \leftrightarrow \mathit{y}_{a,b} \label{eqn:def_gap_aux} 
\end{align}
\end{minipage}
\begin{minipage}{0.45\textwidth}
\begin{align}
     a \land b \land \mathit{y}_{a,b}\ \leftrightarrow \mathit{gap}_{a,b} \label{eqn:def_gap}
\end{align}
\end{minipage}

\medskip
\noindent
where $\mathit{gap}_{a,b}$ is a Boolean 
variable indicating a gap in the ranks of true atoms $a$ and $b$. 

We denote the ranking constraints (\ref{eqn:rank_pos_level})--(\ref{eqn:def_gap}) by $\mathit{TrRk(P)}$; if $P$ is tight, $\mathit{TrRk(P)}=\emptyset$.

\mysubsubsection{Body Translation \bm{$\mathit{TrBd(r)}$}.}
Next, for each $r \in P$, we perform a body translation $\mathit{TrBd(r)}$.
Suppose first that $r$ is a constraint. If $B(r)$ is normal, i.e., 
of form (\ref{eqn:normal}), then we add the clause
\begin{align}\label{eqn:constraint_normal}
\textstyle        \bigvee_{b \in B^+(r)} \neg b \vee \bigvee_{b \in B^-(r)} b\,,
\end{align}
whereas if $B(r)$ is weighted (\ref{eqn:weighted}), we add the constraint
\begin{align}\label{eqn:constraint_weighted}
\textstyle   \sum_{b \in B^+(r)} b \cdot w_b^r + \sum_{b \in B^-(r)} \naf b \cdot w_b^r \ \leq \ l - 1\,.
\end{align}
Note that 
this is a pseudo-Boolean constraint, 
which our intended formalism does not support, and likewise Boolean variables in linear constraints. 
To circumvent this, we introduce new 0-1 integer variables for each literal and link their values; for better readability, we will leave this implicit.
We similarly use auxiliary variables for negated atoms  in conjunctions 
and leave this also implicit.

If $r$ is not a constraint, 
we 
divide $H(r)$ into $T= \{ a \in H(r) \mid \SCC_P(a) \cap B^+(r) = \emptyset \}$ and $H(r)\setminus T$, where 
$T$ are the head atoms that are locally tight.
If $T\neq \emptyset$, we perform the standard Clark's completion~\citep{DBLP:conf/adbt/Clark77} to $r$, i.e., if $B(r)$ is normal, we add 
\begin{align}\label{eqn:tight_body_normal}
\textstyle      \bigwedge_{b \in B^+(r)} b \land \bigwedge_{b \in B^-(r)} \neg b \ \leftrightarrow \ \bd{r}{}\,;
\end{align}
and if $B(r)$ is weighted, we add
\begin{align}\label{eqn:tight_body_weighted}
\textstyle         \sum_{b \in B^+(r)} b \cdot w_b^r + \sum_{b \in B^-(r)} \naf b \cdot w_b^r  \geq  l \leftrightarrow \ \bd{r}{}\,.
\end{align}
Furthermore, for each $a \in T$ such that $|\SCC_P(a)| > 1$, we add the following constraint, which enforces that $a$ has rank~1 if both $a$ and $\bd{r}{}$ are true, where $s_a = |\SCC_P(a)|+1$:
\begin{align}
     s_a \cdot \bd{r}{} \ + s_a \cdot a + \ 1 \cdot \ell_a \ \leq \ 2 \cdot s_a + 1\,,
     \label{eqn:normal_ext_rk_one}
\end{align}
%
and for each $a \in H(r)\setminus T$, 
we add constraints as follows: 
for a normal  $B(r)$ of form  (\ref{eqn:normal}), 
\begin{align}
      \bigwedge_{b \in B^+(r) \setminus \SCC_P(a)} \mysummanddispq{-0.5cm}{b} \land \bigwedge_{b \in B^+(r)\cap  \SCC_P(a)} \mysummanddispq{-0.5cm}{\mathit{dep}_{a,b}} \land\bigwedge_{b \in B^-(r)} b \ \leftrightarrow \ \bd{r}{a} \label{eqn:non_tight_body} \\
      \neg \bd{r}{a} \ \lor\ \bigvee_{b \in B^+(r)\cap \SCC_P(a)} \mysummanddisp{-0.75cm}{ \neg \mathit{gap}_{a,b}}\label{eqn:non_tight_normal_gap}\,,
\end{align}
whereas for a weighted $B(r)$ of form (\ref{eqn:weighted}), we add
\begin{align}
    \sum_{b \in B^+(r) \setminus \SCC_P(a)} \mysummanddisp{-0.75cm}{b\cdot w_b^r} \  + \sum_{b \in B^-(r)}   \mysummanddisp{-0.25cm}{\naf b \cdot w_b^r} \ \geq\ l  \ \leftrightarrow \ \mathit{ext}_r^a \label{eqn:body_non_tight_weighted_ext} \\
    \sum_{b \in B^+(r) \setminus \SCC_P(a)} \mysummanddisp{-0.75cm}{b \cdot w_b^r} \ + \sum_{b \in B^+(r) \cap \SCC_P(a)}  \mysummanddisp{-0.75cm}{ \mathit{dep}_{a,b} \cdot w_b^r} 
    + \sum_{b \in B^-(r)} \mysummanddisp{-0.25cm}{\naf b \cdot w_b^r} \ \geq\ l  \ \leftrightarrow \ \mathit{int}_r^a \label{eqn:body_non_tight_weighted_int} \\
    \sum_{b \in B^+(r) \setminus \SCC_P(a)} \mysummanddisp{-0.75cm}{b \cdot w_b^r} \ + \sum_{b \in B^+(r) \cap \SCC_P(a)} \mysummanddisp{-0.75cm}{\mathit{gap}_{a,b} \cdot w_b^r} 
    + \sum_{b \in B^-(r)} \mysummanddisp{-0.25cm}{\naf b \cdot w_b^r} \ \leq\ l-1  \ \leftrightarrow \ \mathit{aux}_r^a \label{eqn:non_tight_weighted_gap} \\[-3pt]
    \mathit{ext}_r^a \lor \mathit{aux}_r^a \lor \neg \mathit{int}_r^a  \label{eqn:non_tight_weighted_gap_aux} \\
    s_a \cdot \mathit{ext}_r^a \ + s_a \cdot a + \ 1 \cdot \ell_a \ \leq \ 2\cdot s_a + 1 \label{eqn:weighted_ext_rk_one} \\
    \mathit{ext}_r^a \lor \mathit{int}_r^a \ \leftrightarrow \ \bd{r}{a}\,.
    \label{eqn:body_non_tight_weighted_or}
\end{align}

Overall, the rule body translation follows the intuition of the original completion by \cite{DBLP:conf/adbt/Clark77}. Namely, we introduce an auxiliary variable for each rule and constrain it to be true iff the rule body is true.
For each head atom $a$ from the SCC of some body atom,
we follow the approach by \cite{DBLP:journals/corr/abs-2308-15888} and introduce an auxiliary atom $\bd{a}{r}$ for the pair of $a$ and the rule body of $r$. The atom $\bd{a}{r}$ is 
set true exactly when the rule body ``fires'' without 
need of 
cyclic support, which is achieved by considering the dependency variables instead of the atoms, cf.\  
(\ref{eqn:non_tight_body}).
For weighted rule bodies, we follow \cite{DBLP:journals/corr/abs-2308-15888} and introduce auxiliary variables for external (\ref{eqn:body_non_tight_weighted_ext}) and internal (\ref{eqn:body_non_tight_weighted_int}) support of a rule body and a head atom. The former can be seen as the fact that the rule body fires regardless of any atoms in the  SCC of the head atom, while the latter expresses 
rule firing despite some potentially cyclic dependencies.
Constraint (\ref{eqn:body_non_tight_weighted_or}) 
defines an auxiliary variable denoting that the rule supports the head atoms, which is true whenever internal or external support exists.
The constraints  (\ref{eqn:normal_ext_rk_one}), (\ref{eqn:non_tight_normal_gap}), (\ref{eqn:non_tight_weighted_gap_aux}), and (\ref{eqn:weighted_ext_rk_one}) ensure a strict ranking, i.e., no gaps in the level mapping.


\mysubsubsection{Head Translation \bm{$\TrHd(r)$}.}
%
To capture the semantics of a rule $r$, i.e., if $B(r)$ holds then $H(r)$ hold as well, we need further constraints in the translation $\TrHd(r)$.

For each $a\,{\in}\,H(r)$, we use a new Boolean variable $\supp{r}{a}$ to denote that $r$ supports $a$.
Suppose first $r$ is a disjunctive rule and $|H(r)| > 1$. Recall that
by our assumption, every $a \in H(r)$ is locally tight, 
so we only need to consider the single body variable $\mathit{bd}_r$.

Inspired by \citeauthorS{DBLP:conf/ijcai/AlvianoD16} disjunctive completion, we add for each $a_i \in H(r)$:
\begin{align}
\textstyle    \mathit{bd}_r \land \bigwedge_{a_j\in H(r), i \neq j } \neg a_j  \ \  \leftrightarrow \ \ \supp{r}{a} \label{eqn:disjunctive_supp}
\end{align}
Furthermore, we add the following clause ensuring that the rule is satisfied:
\begin{align}
\textstyle    \bigvee_{a \in H(r)} a \ \lor \neg \mathit{bd}_r \label{eqn:disj_rule_sat}
\end{align}
%
%
Otherwise, $r$ is a choice rule or $|H(r)| = 1$. For each $a \in H(r)$  we add the constraint
\nop{******* original ****
\begin{numcases}
 \   \mathit{sp}_r^{a} \leftrightarrow \mathit{bd}_r & \text{ if } $\SCC_P(a) \cap B^+(r) = \emptyset$
 \label{eqn:body_tight_supp_equiv}
 \\
\      \mathit{sp}_r^{a} \leftrightarrow \bd{r}{a} & \text{ otherwise}.\label{eqn:body_non_tight_supp_equiv}
\end{numcases}
****** original *******}
\linebreak
\begin{minipage}{0.55\textwidth}
\begin{align}
     \mathit{sp}_r^{a} \leftrightarrow \mathit{bd}_r & &\text{ if } \SCC_P(a) \cap B^+(r) = \emptyset
 \label{eqn:body_tight_supp_equiv}
 \end{align}
\end{minipage}
\begin{minipage}{0.45\textwidth}
\begin{align}
\text{ and }\quad   \mathit{sp}_r^{a} \leftrightarrow \bd{r}{a} & & \text{ otherwise}. \label{eqn:body_non_tight_supp_equiv}
\end{align}
\end{minipage}
\vspace{0pt}

Note that these constraints define the support variables as the respective rule bodies, and thus would make them 
redundant. However, we keep them to ease readability and for formulating further constraints.
Furthermore, if $r$ is not a choice rule, we add:
\begin{align}
    \supp{r}{a} \rightarrow a \label{eqn:normal_supp}
\end{align}


\mysubsubsection{Supporteness Constraints \bm{$\mathit{TrSupp(P)}$}.}
It remains to encode the supportedness condition of a model.
%
This is achieved by adding for each $a \in \atomsP \setminus \linAtomsP$ the following clause to $\mathit{TrSupp(P)}$:
\begin{align}
 \textstyle    \bigvee_{r \in P, a \in H(r)} \supp{r}{a} \lor \neg a \,.\label{eqn:support_clause}
\end{align}

\subsection{Correctness}

That 
$\Tr(P)$ captures the answer sets of a CASP program $P$ faithfully in a 1-1 correspondence is shown in several steps. We view e-interpretations as models of $\Tr(P)$ with the usual semantics.
The following lemma is useful (cf.\ Def.~\ref{def:partially_shifted} for 
partially shifted programs).
\begin{lemma}\label{lem:trans_mod_rank_supp}
    For every partially shifted HCF program $P$,  if $\langle I , \delta \rangle \models \Tr(P)$ then $\langle I\cap \atomsP, \delta' \rangle$ is a modular ranked scc-supported model of $P$, where $\delta'(\ell_a)=1$ for $a\in I$ s.t. $|\SCC_P(a)|=1$ and $\delta'(\ell_a)=\infty$ for $a \in \atomsP \setminus I$.
\end{lemma}
Based on this lemma and Proposition~\ref{prop:as_equiv_mod_rank_supp}, we obtain that the translation is sound.

\begin{theorem}[Soundness of $\Tr(P)$]
\label{theo:tr-soundness}
 For every partially shifted HCF program $P$, if  $\langle I, \delta \rangle \models \Tr(P)$ then  $\langle I', \delta' \rangle\in \AS(P)$, where $I' = I \cap \atomsP$ and $\delta'(v)=\delta(v)$ for each $v\in \varsP$.
\end{theorem}
Conversely, we show also completeness. 

\begin{theorem}[Completeness of $\Tr(P)$]
\label{theo:tr-completeness}
 For every partially shifted HCF program $P$ and answer set $\langle I , \delta \rangle$ of $P$, there exists some e-interpretation  $\mathcal{I'}=\langle I', \delta' \rangle$ s.t. $I'\cap \atomsP=I\cap \atomsP$, $\delta'(v)=\delta(v)$ for $v\in \varsP$, and $\mathcal{I}' \models \Tr(P)$.
\end{theorem}
Theorems~\ref{theo:tr-soundness} and \ref{theo:tr-completeness} establish 
a many-to-one mapping between the models of the translation and the answer sets of the program. 
That the mapping is in fact 1-1 is achieved through {\em correspondence constraints} given by (\ref{eqn:def_gap_aux}), (\ref{eqn:def_gap}), (\ref{eqn:normal_ext_rk_one}), (\ref{eqn:non_tight_normal_gap}), (\ref{eqn:non_tight_weighted_gap}), (\ref{eqn:non_tight_weighted_gap_aux}), (\ref{eqn:weighted_ext_rk_one}), and the gap variables, which---as for \cite{DBLP:journals/corr/abs-2308-15888}---ensure that the level mapping is \emph{strict}, i.e., has no gaps and starts at 1.
\begin{lemma}\label{lem:trans_strict}
    Suppose  $P$ is a partially shifted HCF program and $\mathcal{I} = \langle I, \delta \rangle$, $\mathcal{I}' = \langle I', \delta' \rangle$ 
    are models of $\Tr(P)$. Then 
   $I\cap \atomsP=I'\cap \atomsP$ implies $\delta(\ell_a) = \delta'(\ell_a)$ for every $a \in \atomsP$.
\end{lemma}

\begin{theorem}[1-1 model correspondence between $P$ and $\Tr(P)$]
\label{theo:tr-1-1-correspondence}
    For a partially shifted HCF program $P$, 
    $\AS(P)$ corresponds 1-1 to the  models of $\Tr(P)$.
\end{theorem}

For the implementation and the experiments, we also consider a non-strict version of the translation without the mentioned constraints, where Theorem~\ref{theo:tr-1-1-correspondence} does not hold.

\vspace*{-\baselineskip}

\section{Implementation}

The 
translation $\Tr(P)$ is available via  the tool \solver{}, which is implemented in Rust\footnote{\url{https://www.rust-lang.org/}} 
the source code is online accessible.\footnote{
\url{https://www.kr.tuwien.ac.at/systems/asp-fzn/}
} As mentioned above, $\Tr(P)$, as described, is not in the Integer Programming \emph{standard form}~\citep{wolsey2021}.
However, using well-known transformations and  0-1 variables instead of Booleans, it can be easily cast into 
this form.

The \solver{} tool translates a given CASP program $P$ into a FlatZinc~\citep{DBLP:conf/cp/NethercoteSBBDT07} theory that
has corresponding models. 
Program $P$ can be 
either in ASPIF format~\citep{DBLP:journals/tplp/KaminskiRSW23} as produced by gringo or as a non-ground ASP program, which is then passed on to gringo for grounding.
The FlatZinc theory can then be processed externally or 
relayed by \solver{} via an interface to MiniZinc with
a backend solver as a parameter.
Note that we do no preprocessing of the given ASPIF input, as we generally expect the grounder (for us, gringo), to handle this step and investigating further preprocessing is a topic of future work. 
%



The tool 
supports linear constraints 
similar to the gringo-based CASP solver clingcon~\citep{DBLP:journals/tplp/BanbaraKOS17}, but expects them to occur in rule bodies, and 
further several global constraints,  viz.\ \emph{alldifferent}, \emph{disjunctive}, and \emph{cumulative} constraints.
As for clingcon, these
constraints are specified via gringo's  theory interface \citep{DBLP:journals/tplp/KaminskiRSW23}; see Appendix A for theory definitions.
Minimization objectives over the linear variables are  
akin to those in clingcon,
yet \solver{} 
allows to freely mix  such objectives with plain weak constraints, resp.\ minimization objectives, in ASP.






The \solver{} tool can be run via command line:
{\small
\begin{Verbatim}
 > asp-fzn [OPTIONS] [INPUT_FILES]...
\end{Verbatim}
}
A complete description of the arguments can be found in the appendix or online.
Essentially, \solver{} can be used either as a pure translation tool to convert ASPIF read from stdin into FlatZinc (optionally including an output specification which can be given to MiniZinc), or as a solver by
specifying  
a backend solver for  
MiniZinc, which must be installed on the system.
If a MIP solver is used, 
the translation output is in standard form and no further linearization is needed. By default, \solver{} interprets input ASP files as non-ground programs and uses gringo to first ground them.

\begin{myexample}
\begin{listing}
{\scriptsize
\begin{minipage}{0.45\textwidth}
\begin{Verbatim}
> cat example.lp
{a;b} :- c.
:- 3 <= #sum{1: a; 2: b}.
c :- not d.
&dom{ 0..2 } = x.
&dom{ 0..1 } = y.
d :- &sum{ x ; y } != 3. 
val(x,V) :- &sum{ x } = V, V = 1..2.
val(y,V) :- &sum{ y } = V, V = 1..1.
\end{Verbatim}
\end{minipage}
\nop{**** hide long version ****
\begin{minipage}{0.5\textwidth}
\begin{Verbatim}
> asp-fzn -s cp-sat -a example.lp 
d val(y,1) 
----------
d val(y,1) val(x,1) 
----------
c val(y,1) val(x,2) 
----------
c a val(y,1) val(x,2) 
----------
c b val(y,1) val(x,2) 
----------
d val(x,2) 
----------
d 
----------
d val(x,1) 
----------
\end{Verbatim}
\end{minipage}
***** long version **** }
\begin{minipage}{0.3\textwidth}
\begin{Verbatim}
> asp-fzn -s cp-sat -a example.lp 

d val(y,1) 
----------
d val(y,1) val(x,1) 
----------
c val(y,1) val(x,2) 
----------
c a val(y,1) val(x,2) 
----------
\end{Verbatim}
\end{minipage}
\begin{minipage}{0.2\textwidth}
\begin{Verbatim}


c b val(y,1) val(x,2) 
----------
d val(x,2) 
----------
d 
----------
d val(x,1) 
----------
\end{Verbatim}
\end{minipage}
}

\caption{Running example (left) solved with 
\solver{} (dashed lines separate answer sets)}\label{lst:running_ex_aspfzn}
\end{listing}
Listing~\ref{lst:running_ex_aspfzn} shows the CASP program $P_2$ from Ex.~\ref{ex:running_ctd2} in the language of gringo with the \solver{} theory definition and the output set to enumerate all answer sets.
\end{myexample}


\section{Experiments}

We now demonstrate the effectiveness of \solver{} on 
benchmark problems.
All experiments were run on a cluster with 10 nodes, each 
having 2 Intel Xeon Silver 4314 (16 cores @ 2.40GHz, 24MB cache, no hyperthreading, 2 cores reserved for system, each core can use 1MB L3 cache max.), running Ubuntu 22.04 (Kernel 5.15.0-131-generic), with memory limit 30GB and 
20 min timeout. 
All encodings, instances, and logs are available at {
\dataURL
}.

\subsection{ASP Benchmarks}\label{sec:asp_bench}

We compare \solver{} 0.1.0 with ASP solvers clingo 5.7.1~\citep{DBLP:journals/tplp/GebserKKS19} and DLV 2.1.0~\citep{DBLP:conf/lpnmr/AlvianoCDFLPRVZ17} 
on 
benchmarks from ASP competitions~\citep{calimeri14comp3,alviano13comp4,calimeri16comp5}. 
As backend solvers for \solver{}, we used the MIP solver Gurobi 12.0.1~\citep{gurobi} and CP solvers CP-SAT 9.12.4544 from Google OR-Tools~\citep{perron_et_al:LIPIcs.CP.2023.3} and 
Chuffed 0.13.2~\citep{Chu11}. 

Both CP-SAT and Chuffed are \emph{lazy-clause generation} based, which is a method taken from SMT and has been highly effective for CP solving. In particular, CP-SAT has won the gold medal in the MiniZinc Challenge\footnote{\url{https://www.minizinc.org/challenge/}} for the last years. 
Gurobi on the other hand is a  state-of-the-art, proprietary MIP solver, which has a MiniZinc interface. 
We ran all solvers using default settings, except for CP-SAT (interleaved search enabled).
For  \solver{}, we used gringo 5.7.1
for grounding and MiniZinc 2.9.2~\citep{DBLP:conf/cp/NethercoteSBBDT07} to interface Gurobi and for output formatting, and
 we considered two settings: the strict translation $\Tr(P)$ with a 1-1 mapping between the models of $\Tr(P)$ and $\AS(P)$, and the 
 non-strict
many-to-one variant.

%
\begin{table}
\caption{ASP problems,
       $n$ instances, type {\bf T} \,=\, 
      (o)ptimization $\mid$ (d)ecision, (*) non-tight}
\scriptsize
     \centering
     \begin{minipage}{0.3\textwidth}
     \begin{tabular}{l@{}r@{}c}
     \toprule
         \textbf{Problem Domain} & $n$ & \textbf{T} \\
         \midrule 
         BayesianNL*                          & 60 & $o$ \\[0.5pt]
         BottleFillingProblem                & 20 & $d$ \\[0.5pt]
         CombinedConfiguration*               & 20 & $d$ \\[1pt]
         \parbox[t]{2.7cm}{ConnectedMaximum-\\[-1pt] ~~DensityStillLife*} & 20 & $o$ \\[8pt]
         CrewAllocation                      & 52 & $d$ \\[0.5pt]
         CrossingMinimization                & 20 & $o$ \\[0.5pt]
         GracefulGraphs                      & 20 & $d$ \\[0.5pt]
         GraphColouring                      & 20 & $d$ \\[0.5pt]
         HanoiTower                          & 20 & $d$ \\[0.5pt]
         IncrementalScheduling               & 20 & $d$ \\[1.5pt]
         \bottomrule
     \end{tabular}
     \end{minipage}
     \begin{minipage}{0.33\textwidth}
     \begin{tabular}{l@{}r@{}c}
     \toprule
         \textbf{Problem Domain} & $n$ & \textbf{T} \\
         \midrule 
         KnightTourWithHoles*                 & 20 & $d$ \\[0.5pt]
         Labyrinth*                           & 20 & $d$ \\[0.5pt]
         MarkovNL*                            & 60 & $o$ \\[0.5pt]
         MaxSAT                              & 20 & $o$ \\[0.5pt]
         MaximalCliqueProblem                & 20 & $o$ \\[0.5pt]
         Nomistery                           & 20 & $d$ \\[0.5pt]
         PartnerUnits                        & 20 & $d$ \\[0.5pt]
         \parbox[t]{3.1cm}{PermutationPattern-  \\[-2pt] 
         ~~Matching}          & 20 & $d$ \\[1pt]
         \parbox[t]{3.1cm}{QualitativeSpatial-  \\[-2pt] 
         ~~Reasoning}         & 20 & $d$ \\
         RicochetRobots                      & 20 & $d$ \\[0.5pt]
         \bottomrule
     \end{tabular}
     \end{minipage}
     \begin{minipage}{0.33\textwidth}
     \begin{tabular}{l@{}r@{}c}
     \toprule
         \textbf{Problem Domain} & $n$ & \textbf{T} \\
         \midrule 
         Sokoban                             & 20 & $d$ \\[0.5pt]
         Solitaire                           & 20 & $d$ \\[0.5pt]
         StableMarriage                      & 20 & $d$ \\[0.5pt]
         SteinerTree*                         & 20 & $o$ \\[0.5pt]
         Supertree                           & 60 & $o$ \\[0.5pt]
         SystemSynthesis*                     & 20 & $o$ \\[0.5pt]
         TravelingSalesPerson*                & 20 & $o$ \\[0.5pt]
         ValvesLocationProblem*               & 20 & $o$ \\[0.5pt]
         VideoStreaming                      & 20 & $o$ \\[0.5pt]
         Visit-all                           & 20 & $d$ \\[0.5pt]
         WeightedSequenceProblem             & 20 & $d$ \\[1pt]
         \bottomrule
     \end{tabular}
     \end{minipage}
    
     \label{tab:asp_bench}
 \end{table}

We included both decision and optimization problems in the benchmark, listed in Table~\ref{tab:asp_bench}, with 31 problems and 772 instances in total. 
\nop{****** hide long text
The former consists of the following problem domains: 
\emph{Bottle Filling Problem}, 
\emph{Combined Configuration}, 
\emph{Crew Allocation}, 
\emph{Graceful Graphs},
\emph{Graph Colouring},
\emph{Hanoi Tower},
\emph{Incremental Scheduling},
\emph{Knights Tour with Holes},
\emph{Labyrinth},
\emph{No Mistery},
\emph{Partner Units},
\emph{Permutation Pattern Matching},
\emph{Qualitative Spatial Reasoning},
\emph{Ricochet Robots},
\emph{Sokoban},
\emph{Solitaire},
\emph{Stable Marriage},
\emph{Visit-All}, and
\emph{Weighted Sequence Problem}.
The following optimization problems were considered:
\emph{Bayesian Network Learning},
\emph{Connected Max. Density Still Life},
\emph{Crossing Minimization},
\emph{Markov Network Learning},
\emph{MaxSAT},
\emph{Max. Clique},
\emph{Steiner Tree},
\emph{Supertree},
\emph{System Synthesis},
\emph{Travelling Salesperson Problem},
\emph{Valves Location Problem}, and
\emph{Video Streaming}.
Whenever available, we used the 20 instances for each domain which were also used in the ASP competitions. For \emph{Bayesian Network Learning}, \emph{Markov Network Learning} and \emph{Supertree}, 60 instances were considered and 52 \emph{Crew Allocation}.
The total number of instances in the benchmark is thus 772.
*** end hide text ****}
We used the encodings from the competition, 
but replaced in few 
some parts with 
modern 
constructs like choice rules.
Note that the decision variants of all problems, except \emph{StableMarriage}, are NP-hard
and several 
encodings are non-tight.

\nop{******** original table *******
\begin{table}[!ht]
    \centering
    \footnotesize
    \begin{tabular}{lrrrr}
        \toprule
         & \multicolumn{2}{c}{single thread} & \multicolumn{2}{c}{8 threads} \\
         & \emph{Score1} & \emph{Score2} & \emph{Score1} & \emph{Score2} \\
        \midrule
        asp-fzn (CP-SAT)                & 1840.0            & 1888.3            & 2025.0           & 2051.7 \\
        asp-fzn (CP-SAT, non-strict)    & 1871.7            & 1978.3            & 2051.7           & 2101.7 \\
        asp-fzn (Chuffed)               & 327.4             & 327.4             & -- & -- \\
        asp-fzn (Chuffed, non-strict)   & 362.4             & 362.4             & -- & --\\
        asp-fzn (Gurobi)                & 855.0             & 866.7             & 986.7            & 990.0 \\
        asp-fzn (Gurobi, non-strict)    & 935.0             & 960.0             & 1033.3           & 1041.7 \\
        clingo                          & \bfseries 1890.4  & \bfseries 1992.1  & \bfseries 2351.2 & \bfseries 2511.2 \\
        DLV                             & 1524.4            & 1604.4            & -- & -- \\
        \bottomrule
    \end{tabular}
    \caption{Comparison of \solver{} with ASP solvers on plain ASP benchmarks}
    \label{tab:asp_comp}
\end{table}
****** original table *****}
\begin{table}
    
    \centering
    \caption{Comparison of \solver{} with ASP solvers on plain ASP benchmarks. The symbols next to the score indicate whether a higher value ($\uparrow$) or lower value ($\downarrow$) is better.}
    \vspace{-2mm}
    \footnotesize
    \begin{tabular}{lrrr}
        \toprule
         & \multicolumn{3}{c}{single thread}  \\
         & \multicolumn{1}{c}{\emph{Score1}$\uparrow$} & \multicolumn{1}{c}{\emph{Score2}$\uparrow$} & \multicolumn{1}{c}{\new{\emph{Score3}$\downarrow$}}   \\
        \midrule
        asp-fzn (CP-SAT)  / (CP-SAT, non-strict)                & 1840.0  / 1871.7            & 1888.3 / 1978.3   &   \new{153738.5 / 149807.9}    \\
        asp-fzn (Chuffed)  / (Chuffed, non-strict)            & 782.4 / 812.4             & 782.4 / 812.4      &  \new{279592.2 / 275942.3}  \\
        asp-fzn (Gurobi)  / (Gurobi, non-strict)               & 1185.0 /  1265.0             & 1196.7 / 1290.0      &   \new{231543.2 / 222057.8}       \\
         clingo                          & \bfseries 1890.4  & \bfseries 1992.1 & \bfseries \new{147786.8}   \\
        DLV                             & 1524.4            & 1604.4    &    \new{191445.8}  \\
        \midrule
        & \multicolumn{3}{c}{8 threads} \\
        & \multicolumn{1}{c}{\emph{Score1}$\uparrow$} & \multicolumn{1}{c}{\emph{Score2}$\uparrow$} & \multicolumn{1}{c}{\new{\emph{Score3}$\downarrow$}} \\
        \midrule
        asp-fzn (CP-SAT)  / (CP-SAT, non-strict)                 & 2025.0 / 2051.7            & 2051.7 / 2101.7 & \new{131003.7 / 128072.7} \\
        asp-fzn (Gurobi)  / (Gurobi, non-strict)                & 1441.7 / 1478.3             & 1445.0 / 1486.7 & \new{201661.7 / 196921.8} \\
         clingo                         & \bfseries 2351.2 & \bfseries 2511.2 & \bfseries \new{92028.5} \\
        \bottomrule
    \end{tabular}
    
    \label{tab:asp_comp}
\end{table}

Table~\ref{tab:asp_comp} presents the comparison of \solver{} with clingo and DLV, and cactus plots can be found in Appendix A. Here
$\mathit{Score1} = \sum_{i=1}^{31} c_i/n_i * 100$ where $c_i$
is the number of closed instances of domain $D_i$, i.e., shown to be (un)satisfiable for type $d$ resp.\ optimal for type $o$; the maximum score is 3100. 
%
\emph{Score2} measures  the best performers, by
$\mathit{Score2} = \sum_{i=1}^{31} b_i/n_i * 100$, where $b_i$
is the number of instances from 
$D_i$ where the solver either closed the instance or found a solution 
of best value among all solvers.

Lastly, $\mathit{Score3} = \sum_{i=1}^{31} t_i/n_i$ is the \emph{PAR10} score, where $t_i$ is the time the solver took to complete instance $i$ respectively $10\times 1200$ if the solver did not complete the instance. Hence, here a lower number is better.

In single-threaded mode, clingo performs best on \emph{Score1}, but \solver{} with CP-SAT as backend is trailing closely behind, beating DLV. Under the non-strict translation,
\solver{} performs slightly better on \emph{Score1} and significantly better on \emph{Score2} 
. 
Furthermore, clingo also has the best \emph{Score3}, indicating it is also closing most instances quicker than the rest; however, 
\solver{} with CP-SAT under the non-strict translation is only 1.37\% worse than 
clingo.
Gurobi and Chuffed as backends perform worse than CP-SAT,  but the non-strict variant is also better here.
This difference between strict and non-strict variants is similar to previous observations for translation-based ASP solving~\citep{DBLP:conf/lpnmr/JanhunenNS09}.
It seems non-strictness does not interfere with search-tree pruning.

For space reasons, we cannot give a detailed breakdown of the results over the particular problem domains, but unsurprisingly \solver{} performs worse than clingo mostly on domains which are non-tight or feature heavy usage of disjunctions.
An exception here is the Traveling Salesperson Problem where \solver{} using CP-SAT or Gurobi outperforms clingo.
Except for a few further non-tight domains, like Bayesian Network Learning and Systems Synthesis, Gurobi achieves worse results than CP-SAT as a backend solver.

Since clingo, Gurobi, and CP-SAT support parallel solving, we ran the benchmark on them using 8 threads. Again, clingo was best, cf.\
Table~\ref{tab:asp_comp}; 
while \solver{} performed better with Gurobi and CP-SAT,
the gap to clingo widened.
Nonetheless, the benchmarks show that \solver{} with the right backend solver is competitive with known ASP solvers.%

\subsection{CASP Benchmarks}

We now turn 
our attention 
to CASP. 
We look at three problem domains with ASP benchmark instances from the literature 
that can be modeled with CASP.
We compare \solver{} against clingcon~5.2.1~\citep{DBLP:journals/tplp/BanbaraKOS17} as it supports a similar language. 


\myparagraph{Parallel Machine Scheduling Problem (PMSP)} was first studied with 
ASP by \cite{DBLP:journals/tplp/EiterGMOSS23}, who provided a benchmark set of 500 instances.
The task is assigning jobs with release dates and sequence-dependent setup times to capable machines. The objective
is minimizing the total makespan, i.e., the maximal completion time of any job.




\begin{table}
\centering
    \footnotesize
 \nop{******* original table ********
 \begin{tabular}{lrrrr}
        \toprule
        & \multicolumn{2}{c}{single thread} & \multicolumn{2}{c}{8 threads}                             \\
        & \emph{closed}                     & \emph{best} & \emph{closed} & \emph{best}                 \\
        \midrule
        asp-fzn (CP-SAT)                & \bfseries 40  & 140           & \bfseries 55  & 155           \\
        asp-fzn (CP-SAT, non-strict)    & \bfseries 40  & \bfseries 166 & 54            & 167           \\
        asp-fzn (Chuffed)               & 0             & 0             & --            & --            \\
        asp-fzn (Chuffed, non-strict)   & 20            & 20            & --            & --            \\
        asp-fzn (Gurobi)                & 26            & 26            & 28            & 36            \\
        asp-fzn (Gurobi, non-strict)    & 27            & 29            & 28            & 41            \\
        clingcon                        & 36            & 36            & 31            & \bfseries 298 \\
        \bottomrule
    \end{tabular}
*********** original table }
\caption{\solver{} vs.\ clingcon on PMSP (strict / non-strict).}
\vspace{-2mm}
\setlength{\tabcolsep}{1.7pt}
    \begin{tabular}{lrrrrrr}
        \toprule
         & \multicolumn{3}{c}{single thread} & \multicolumn{3}{c}{8 threads} \\
         & \multicolumn{1}{c}{\emph{closed}} & \multicolumn{1}{c}{\emph{best}} & \multicolumn{1}{c}{\new{\emph{PAR10}}}  & \multicolumn{1}{c}{\emph{closed}} & \multicolumn{1}{c}{\emph{best}} & \multicolumn{1}{c}{\new{\emph{PAR10}}}  \\
        \midrule
        asp-fzn (CP-SAT) 
        & {\bfseries 40} / \bfseries 40  & 140  / \bfseries 166   &    \new{{\bfseries 11050.6} / 11051.4}  & {\bfseries 55} / 54  & 155 / 167    &  \new{{\bfseries 10699.0} / 10719.5 }   \\
        asp-fzn (Chuffed) 
        & 18 / 20            & 18 / 20   &  \new{11570.5 / 11525.0} & --         & --          &  \new{--}          \\
        asp-fzn (Gurobi) 
        & 26 / 27          & 26 / 29     &  \new{11379.1 / 11355.8}  & 28  / 28           & 36 / 41   & \new{ 11330.7 / 11330.4}        \\
          clingcon                        & 36            & 36       &  \new{11147.8}   & 31            & \bfseries 298 & \new{11264.0} \\
        \bottomrule
    \end{tabular}
  
\label{tab:pmsp}   
\medskip

 \caption{\solver{} vs.\ clingcon on TLSPS.}
 \vspace{-2mm}
    \footnotesize
\noindent\begin{tabular}{lrrrrrr}
        \toprule
        & \multicolumn{3}{c}{single thread} & \multicolumn{3}{c}{8 threads} \\
        & \emph{closed} & \emph{best} & \new{\emph{PAR10}} & \emph{closed} & \emph{best}& \new{\emph{PAR10}} \\
        \midrule
        asp-fzn (CP-SAT)                & \bfseries 55  & \bfseries 76  & \new{\bfseries 6741.5} & 64  & 76 & \new{5850.1} \\
        asp-fzn (Chuffed)               & 11            & 11          & \new{10940.0} & --            & -- & \new{--}           \\
        clingcon                        & 7            & 22        &  \new{11329.1}   &  \bfseries 
 77            &  \bfseries  90  &   \new{\bfseries 4553.3}       \\
        \bottomrule
    \end{tabular}

    \label{tab:tlsps}

    \medskip
    \footnotesize
    \caption{\solver{} vs.\ clingcon on MAPF.}
    \vspace{-2mm}
    \begin{tabular}{lrrrr}
        \toprule
        & \multicolumn{2}{c}{single thread} & \multicolumn{2}{c}{8 threads}   \\
        & \emph{closed} & \new{\emph{PAR10}} & \emph{closed} & \new{\emph{PAR10}} \\
        \midrule
        asp-fzn (CP-SAT)                & \bfseries 224 & \new{\bfseries 7116.6} &  \bfseries 233  & \new{\bfseries 6913.4}          \\
        asp-fzn (Chuffed)               & 159           & \new{8553.7} & -- & \new{--} \\
        asp-fzn (Gurobi)                & 194           & \new{7766.5} &  194  & \new{7762.9}          \\
        clingcon                        & 177           & \new{8138.1} &  209  & \new{7428.2} \\
        \bottomrule
    \end{tabular}
   
    \label{tab:mapf}
\end{table}

Table~\ref{tab:pmsp} shows the results for PMSP 
on the 500 instances using the (non-tight) CASP encoding which for space reasons is given in the appendix.
%
In single-threaded solving, \solver{} with CP-SAT and the non-strict translation is again superior, closing 40 instances and achieving the best result for 166; the strict translation is slightly worse but closes the same number of instances. 
The solver clingcon closed 36 instances, which is more than \solver{} with any of the other backend solvers.

\new{Looking at the \emph{PAR10} score, cf.\  Section~\ref{sec:asp_bench}, we see that \solver{} with CP-SAT achieves the best score, indicating that it can close the instances faster than clingcon.
Interestingly, the strict translation does better here but the difference is marginal.}


The picture changes for multi-threaded solving: here clingcon achieved the top value for \emph{best} with 298 instances vs.\ 167 by \solver{} with CP-SAT for the non-strict translation. The latter setting closed the second most instances 
(54);
changing to the strict translation closed one instance but decreased best results. The large number of best results found by clingcon can be explained by its strength in finding feasible solutions for PMSP in parallel mode, while \solver{} struggles. However, when a solution is found, \solver{} and CP-SAT typically provide the best final result 
\new{and as the \emph{PAR10} score shows, it also takes the least CPU time to prove optimality}.


\myparagraph{Test Laboratory Scheduling Problem (TLSPS)} 
is a
variant of
a scheduling problem due to \cite{Mischek2018TechReport} that is efficiently solvable using a CASP encoding~\citep{geibinger2021constraint,DBLP:journals/ai/EiterGRMOPS24}. As the encoding is tight, the strict and the non-strict translation are the same.

TLSPS concerns scheduling jobs in a test lab by assigning them an execution mode, a starting time in its time window, and required resources from a set of qualified resources. 
The overall objective has several components, like assigning  preferred employees for certain jobs, minimizing the number of employees on a project, reducing tardiness, and minimizing the project duration.

For clingcon, we essentially use  \citeauthorS{DBLP:journals/ai/EiterGRMOPS24} encoding employing ASP minimization.
The \solver{} 
encoding, shown partially in Listing~\ref{lst:tlsps} (full version in Appendix A), mixes minimization of plain ASP and linear variables; clingcon does not support the latter, but allows for a more natural encoding of the objective.
Also, the \solver{} encoding uses global disjunctive constraints to enforce unary resource usage; this is not possible in clingcon but proved to be quite effective.

\begin{listing}
    {\scriptsize
\begin{Verbatim}
&dom{R..D} = start(J) :- job(J), release(J, R), deadline(J, D).
&dom{R..D} = end(J) :- job(J), release(J, R), deadline(J, D).
&dom{L..H} = duration(J) :- job(J), L = #min{ T : durationInMode(J, _, T) }, 
                            H = #max{ T : durationInMode(J, _, T) }.              
1 {modeAssign(J, M) : modeAvailable(J, M)} 1 :- job(J).
:- job(J), modeAssign(J, M), durationInMode(J, M, T), &sum{ duration(J) } != T.
:- job(J), &sum{end(J); -start(J); -duration(J)} != 0.
:- precedence(J,K), &sum{start(J); -end(K)} < 0 .
   ...
&disjoint{ start(J)@duration(J) : workbenchAssign(J,W) } :- workbench(W).
&disjoint{ start(J)@duration(J) : empAssign(J,W) } :- employee(W).
&disjoint{ start(J)@duration(J) : equipAssign(J,W) } :- equipment(W).

#minimize{1,E,J,s2 : job(J), empAssign(J, E), not employeePreferred(J, E)  }. 
#minimize{1,E,P,s3 : project(P), empAssign(J, E), projectAssignment(J, P)}.
&dom{0..H} = delay(J)  :- job(J), horizon(H).
:- job(J), due(J, T), &sum{end(J)} > T, &sum{-1*delay(J); end(J)} != T.
:- job(J), due(J, T), &sum{end(J)} <= T, &sum{delay(J)} != 0.
&minimize{delay(J) : job(J)}.
   ...
\end{Verbatim}
}
\vspace{-4mm}
\caption{Partial TLSPS encoding used by \solver{}.}\label{lst:tlsps}
\end{listing}

\nop{**** original table
\begin{table}[!ht]
    \centering
    \footnotesize
    \begin{tabular}{lrrrr}
        \toprule
        & \multicolumn{2}{c}{single thread} & \multicolumn{2}{c}{8 threads} \\
        & \emph{closed} & \emph{best} & \emph{closed} & \emph{best} \\
        \midrule
        asp-fzn (CP-SAT)                & \bfseries 55  & \bfseries 76 & 64  & 76  \\
        asp-fzn (Chuffed)               & 11            & 11            & --            & --            \\
        clingcon                        & 7            & 22            &  \bfseries 
 77            &  \bfseries  90            \\
        \bottomrule
    \end{tabular}
 \caption{\solver{} vs.\ clingcon on TLSPS}
    \label{tab:tlsps}
\end{table}
***** original table *******}

Our benchmark consisted of 123 instances from \cite{Mischek2018TechReport} of which 3 are real-world; the instances were converted to ASP facts (see supplementary data).

The results, collected in Table~\ref{tab:tlsps}, show that \solver{} performed very well. 
Column \emph{closed} lists how many 
instances were solved and proven optimal, and \emph{best} lists the number of solutions 
that were best among all solvers; instances for which no solver found any solution 
were discarded. 
%
Our tool \solver{} with backend CP-SAT performed best for TLSPS in single-threaded mode as it solved 55 instances to optimality and produced for 76 instances the best result. 
\new{Furthermore, it also achieved the lowest, and thus best, \emph{PAR10} score.}
With backend Chuffed, \solver{} performed significantly worse but produced always best results; also clingon 
lagged significantly behind. Gurobi was not used as it does not support disjunctive global constraints.
%
%
With multi-threaded solving, clingcon outperformed \solver{} and CP-SAT, closing more instances and more 
often yielding the best result\new{, while also taking less time to prove optimality on average}.


\myparagraph{Multi Agent Path Finding (MAPF)} 
was recently studied by \cite{DBLP:journals/corr/abs-2403-12153}, who provided an instances and a generator. The task is planning the routes of several agents 
to reach their goals without colliding.
%
%
%
%
%
Our tight CASP encoding (see Appendix A) is similar to
\citeauthorS{DBLP:journals/corr/abs-2403-12153} but uses linear constraints for the event ordering. 

\nop{ **** original table
\begin{table}[!ht]
    \centering
    \footnotesize
    \begin{tabular}{lrr}
        \toprule
        & single thread & 8 threads   \\
        \midrule
        asp-fzn (CP-SAT)                & \bfseries 224 &  \bfseries 233            \\
        asp-fzn (Chuffed)               & 159             & --\\
        asp-fzn (Gurobi)                & 194           &  194            \\
        clingcon                        & 177           &  209  \\
        \bottomrule
    \end{tabular}
   \caption{\solver{} vs.\ clingcon on MAPF}
    \label{tab:mapf}
\end{table}
******* hide original table ******}

For our comparison,
we selected 547 MAPF instances from one of the sets by \citeauthor{DBLP:journals/corr/abs-2403-12153}. The results are shown in Table~\ref{tab:mapf}, listing the number of instances for which a plan was found 
(MAPF has no optimization objective). 
%
With Gurobi and CP-SAT as backends, \solver{} closed more instances than clingcon, but it closed fewer with Chuffed.
The best result is achieved by \solver{} and CP-SAT with 224 instances solved\new{; it also achieves the best \emph{PAR10} score}.
%
%
For parallel solving (8 threads), \solver{} with CP-SAT closed the most instances (233, 9 more than single-threaded). Gurobi did not benefit from parallelism while it improved the clingcon results. However, the latter still lagged behind CP-SAT.

\subsection{Summary}

Overall,  \solver{} with CP-SAT as  backend achieved decent results, being competitive as a plain ASP solver and performing better than clingcon for TLSPS and MAPF.
However, we note that CP-SAT  
has a rather high memory footprint.
The average total memory usage of clingo on the plain ASP benchmark was five times lower than the one of \solver{} with CP-SAT and the latter hit the memory limit for several instances.
This is not only due to the translation itself, but a high memory usage of CP-SAT in general.


Regarding strict vs.\  non-strict translation, it appears beneficial to use the non-strict translation by default, except when solution enumeration is requested. 
The time it takes to translate the gringo output to FlatZinc, this never took longer than a couple of seconds and was dwarfed by the grounding time.

\section{Related Work and Conclusion}

        

For a thorough survey of CASP solvers, we refer to \citeauthorS{DBLP:journals/tplp/Lierler23} survey.
Closest related to \solver{} is
clingcon~\citep{DBLP:journals/tplp/BanbaraKOS17}
as it features a similar language and is  based on clingo~\citep{DBLP:journals/tplp/GebserKKS19}.
Notably, while clingcon supports some global constraints, their usage is often limited. E.g.\ 
variables occur in disjunctive constraints
 unconditionally, i.e., whether a linear variable is active 
depends only on the truth of atoms determined at grounding time. This excludes disjunctive constraints as used for TLSPS in \solver{}.
Further, clingcon lacks cumulative constraints and disallows mixing ASP minimization and minimization over linear variables.
Closely related to clingcon is clingo-dl~\citep{DBLP:journals/tplp/JanhunenKOSWS17}, 
which is not a full CASP solver as it only supports \emph{difference constraints}, a special type of linear constraint. As we consider unrestricted linear constraints,  we did not feature clingo-dl in the evaluation.

EZSMT+~\citep{DBLP:journals/corr/abs-1905-03334} is also a translation-based CASP solver but targets SMT. As it does not support optimization, we did not feature it in the comparison. As a further impediment to a direct comparison, EZSMT+ uses the language of 
EZCSP \citep{DBLP:conf/lpnmr/Balduccini11}, which is quite different from \solver{} and clingcon's theory language.
In difference to clingcon and EZSMT+, EZCSP has slightly different semantics, as the linear constraints are evaluated for each answer set
that may be pruned on violation.
\new{Another translation-based CASP solver is mingo~\citep{DBLP:conf/kr/LiuJN12}, which translates a CASP program into MIP. While mingo does feature optimization, it also differs in language from \solver{} and was not compatible with Gurobi.
}

 As for translation-based plain ASP, our approach borrows heavily from \cite{DBLP:journals/corr/abs-2308-15888} and \cite{DBLP:conf/ijcai/AlvianoD16}. \citeauthor{DBLP:journals/corr/abs-2308-15888}  extended the level mapping formulation to programs with weight rules but provided no implementation,  while \citeauthor{DBLP:conf/ijcai/AlvianoD16}  introduced completion for disjunctive rules not as a translation-based approach per se but
 for DLV~\citep{DBLP:conf/lpnmr/AlvianoCDFLPRVZ17}.
Finally, \citeauthorS{DBLP:conf/ijcai/RankoohJ24}
translation of ASP into MIP relies on prior normalization and an acyclicity transformation 
that explicitly represents dependencies among atoms
by auxiliary variables and encodes supported models; answer sets are 
obtained by adding acyclicity constraints.

\myparagraph{Outlook.} A promising avenue for future work is the investigation of vertex elimination, as used by \citeauthor{DBLP:conf/ijcai/RankoohJ24} in their translation. While it does not guarantee a 1-1 correspondence, it has shown potential for improving performance on standard ASP optimization benchmarks. Additional directions for future research include incorporating more global constraints or exploring novel language constraints that can be modeled in FlatZinc. Another possibility is evaluating metaheuristic FlatZinc solvers, such as using CP-SAT as a purely local-search-based solver.
\new{
Finally, CASP semantics was aligned more with stable reasoning, 
moving away from interpreting linear constraints classically, 
in \citep{DBLP:conf/ijcai/CabalarKOS16,DBLP:conf/ecai/CabalarFSW20,DBLP:journals/tplp/EiterK20}.
A modified translation modeling those semantics would be another highly interesting avenue for future work.
}


\section*{Acknowledgements}

This work was supported by funding from the Bosch Center for AI at Renningen, Germany. 
Tobias Geibinger is a recipient of a DOC Fellowship of the Austrian Academy of Sciences at the Institute of Logic and Computation at the TU Wien.

\bibliographystyle{abbrvnat}
\bibliography{refs}

\newpage
\appendix

\section{Implementation and Experiments}
\label{app:imp}

\subsection{Detailed Results on ASP Benchmarks}

Table \ref{tab:asp_bench_st_details} shows the detailed results for the ASP Benchmark from Section~\ref{sec:asp_bench} for single-threaded solving and Table \ref{tab:asp_bench_mt_details} for 8 threads. 
\begin{figure}
\begin{subfigure}[b]{0.5\textwidth}
        \centering
        \includegraphics[width=\linewidth]{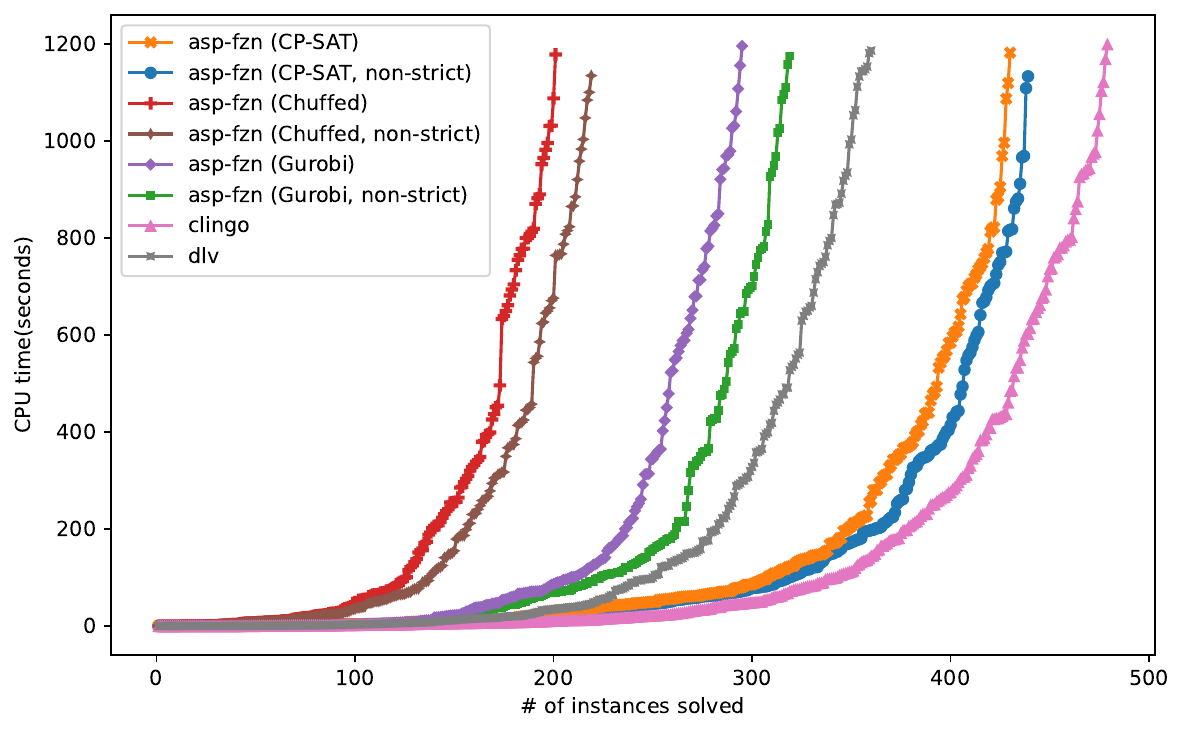}
\end{subfigure}%
\begin{subfigure}[b]{0.5\textwidth}
        \centering
        \includegraphics[width=\linewidth]{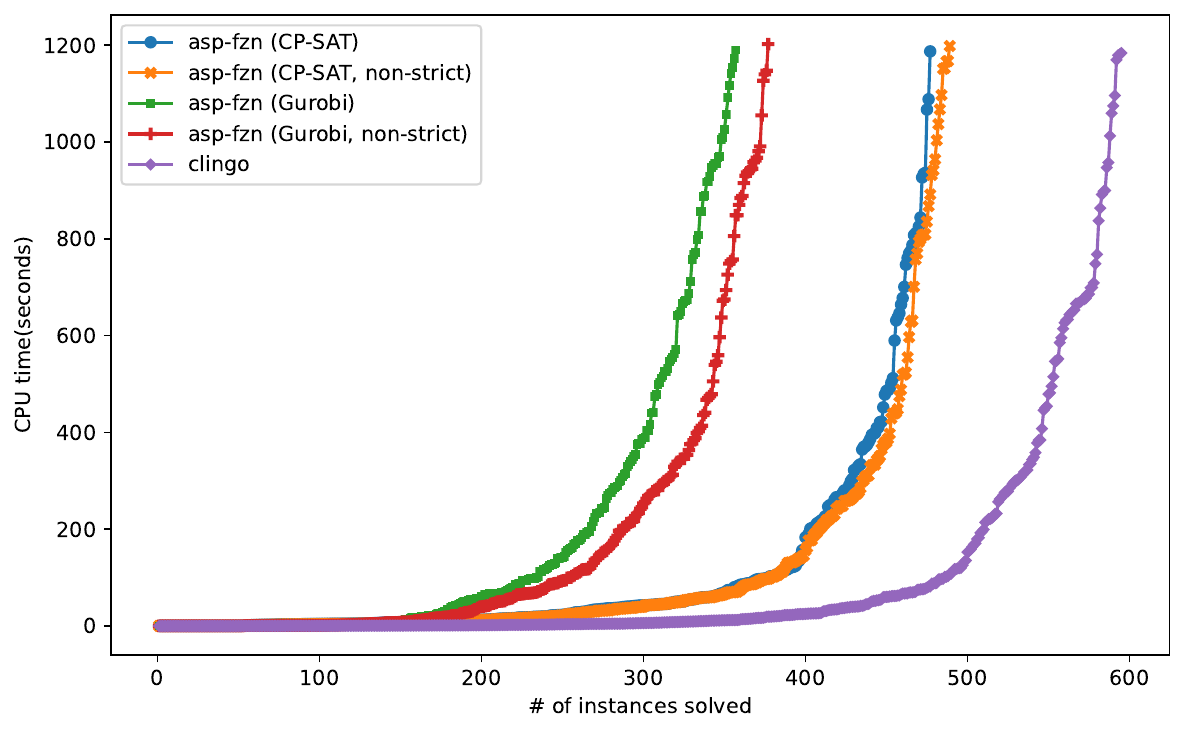}
\end{subfigure}%
    \caption{Cactus plots for solver performance on ASP problems: 1 thread (left) vs.\ 8 threads (right)}
    \label{fig:cactus}
\end{figure}
Figure~\ref{fig:cactus} shows cactus plots for solver performance.

\begin{sidewaystable}
     \setlength{\tabcolsep}{1pt}
     \fontsize{4.6}{8} \selectfont
     \centering
     \caption{Detailed Results for ASP benchmark problems (single-threaded)}

     \begin{tabular}{lrrrrrrrrrrrrrrrr}
     \toprule
         \textbf{Problem Domain} & \multicolumn{9}{c}{asp-fzn} & \multicolumn{3}{c}{clingo} & \multicolumn{3}{c}{DLV}  \\
         & \multicolumn{3}{c}{CP-SAT strict / non-strict} & \multicolumn{3}{c}{Chuffed strict / non-strict} & \multicolumn{3}{c}{Gurobi strict / non-strict} & &  \\
         & \multicolumn{1}{c}{\emph{score1}} & \multicolumn{1}{c}{\emph{score2}} & \multicolumn{1}{c}{\emph{score3}} & \multicolumn{1}{c}{\emph{score1}} & \multicolumn{1}{c}{\emph{score2}} & \multicolumn{1}{c}{\emph{score3}} & \multicolumn{1}{c}{\emph{score1}} & \multicolumn{1}{c}{\emph{score2}} & \multicolumn{1}{c}{\emph{score3}} & \multicolumn{1}{c}{\emph{score1}} & \multicolumn{1}{c}{\emph{score2}} & \multicolumn{1}{c}{\emph{score3}} & \multicolumn{1}{c}{\emph{score1}} & \multicolumn{1}{c}{\emph{score2}} & \multicolumn{1}{c}{\emph{score3}} & \\
         \midrule 
BayesianNL & 40.0 / 46.7 & 46.7 / 61.7 & 7262.2 / 6470.8 & 11.7 / 38.3 & 11.7 / 38.3 & 10646.1 / 7491.9 & 53.3 / 60.0 & 60.0 / 75.0 & 5762.7 / 4873.0 & 71.7 & 80.0 & 3488.9 & 61.7 & 61.7 & 4653.6 \\
BottleFillingProblem & 100.0 & 100.0 & 24.1 & 55.0 & 55.0 & 5420.9 & 100.0 & 100.0 & 43.3 & 100.0 & 100.0 & 3.0 & 100.0 & 100.0 & 3.1 \\
CombinedConfiguration & 10.0 / 10.0 & 10.0 / 10.0 & 10802.1 / 10806.8 & 0.0 / 0.0 & 0.0 / 0.0 & 12000.0 / 12000.0 & 10.0 / 10.0 & 10.0 / 10.0 & 10875.5 / 10890.0 & 50.0 & 50.0 & 6135.9 & 5.0 & 5.0 & 11437.4 \\
ConnectedMaximim-densityStillLife & 45.0 / 50.0 & 50.0 / 100.0 & 6712.2 / 6045.7 & 5.0 / 5.0 & 5.0 / 5.0 & 11438.9 / 11427.8 & 30.0 / 35.0 & 30.0 / 35.0 & 8566.1 / 7983.7 & 35.0 & 35.0 & 7906.0 & 50.0 & 50.0 & 6045.5 \\
CrewAllocation & 100.0 & 100.0 & 0.1 & 80.8 & 80.8 & 2477.4 & 100.0 & 100.0 & 0.2 & 90.4 & 90.4 & 1261.7 & 57.7 & 57.7 & 5204.1 \\
CrossingMinimization & 95.0 & 100.0 & 605.0 & 30.0 & 30.0 & 8456.4 & 95.0 & 100.0 & 614.5 & 35.0 & 35.0 & 7844.4 & 95.0 & 95.0 & 603.5 \\
GracefulGraphs & 55.0 & 55.0 & 5535.4 & 30.0 & 30.0 & 8470.8 & 0.0 & 0.0 & 12000.0 & 55.0 & 55.0 & 5516.7 & 40.0 & 40.0 & 7274.8 \\
GraphColouring & 55.0 & 55.0 & 5552.8 & 15.0 & 15.0 & 10325.6 & 10.0 & 10.0 & 10818.4 & 85.0 & 85.0 & 1936.3 & 55.0 & 55.0 & 5591.7 \\
HanoiTower & 100.0 & 100.0 & 17.8 & 100.0 & 100.0 & 89.4 & 20.0 & 20.0 & 9653.6 & 100.0 & 100.0 & 18.4 & 100.0 & 100.0 & 8.0 \\
IncrementalScheduling & 60.0 & 60.0 & 4936.8 & 25.0 & 25.0 & 9092.4 & 15.0 & 15.0 & 10317.0 & 70.0 & 70.0 & 3652.8 & 45.0 & 45.0 & 6663.6 \\
KnightTourWithHoles & 10.0 / 10.0 & 10.0 / 10.0 & 10804.3 / 10804.1 & 0.0 / 0.0 & 0.0 / 0.0 & 12000.0 / 12000.0 & 40.0 / 40.0 & 40.0 / 40.0 & 7251.2 / 7238.6 & 55.0 & 55.0 & 5457.8 & 45.0 & 45.0 & 6623.6 \\
Labyrinth & 45.0 / 50.0 & 45.0 / 50.0 & 6679.8 / 6118.4 & 20.0 / 15.0 & 20.0 / 15.0 & 9702.7 / 10232.4 & 0.0 / 0.0 & 0.0 / 0.0 & 12000.0 / 12000.0 & 70.0 & 70.0 & 3691.8 & 60.0 & 60.0 & 4947.1 \\
MarkovNL & 3.3 / 3.3 & 26.7 / 16.7 & 11612.4 / 11610.3 & 0.0 / 3.3 & 0.0 / 3.3 & 12000.0 / 11622.8 & 0.0 / 13.3 & 0.0 / 13.3 & 12000.0 / 10448.1 & 55.0 & 90.0 & 5508.4 & 0.0 & 5.0 & 12000.0 \\
MaxSAT & 100.0 & 100.0 & 74.4 & 0.0 & 0.0 & 12000.0 & 85.0 & 85.0 & 1952.1 & 35.0 & 35.0 & 7810.9 & 90.0 & 70.0 & 1250.1 \\
MaximalCliqueProblem & 75.0 & 75.0 & 3168.5 & 0.0 & 0.0 & 12000.0 & 85.0 & 85.0 & 1865.6 & 0.0 & 10.0 & 12000.0 & 5.0 & 5.0 & 11438.1 \\
Nomistery & 50.0 & 50.0 & 6074.4 & 0.0 & 0.0 & 12000.0 & 0.0 & 0.0 & 12000.0 & 40.0 & 40.0 & 7267.9 & 40.0 & 40.0 & 7301.7 \\
PartnerUnits & 70.0 & 70.0 & 3721.5 & 10.0 & 10.0 & 10801.0 & 0.0 & 0.0 & 12000.0 & 70.0 & 70.0 & 3619.5 & 55.0 & 55.0 & 5632.5 \\
PermutationPatternMatching & 35.0 & 35.0 & 7837.5 & 40.0 & 40.0 & 7282.7 & 25.0 & 25.0 & 9040.3 & 80.0 & 80.0 & 2629.6 & 20.0 & 20.0 & 9635.5 \\
QualitativeSpatialReasoning & 25.0 & 25.0 & 9095.8 & 5.0 & 5.0 & 11436.7 & 0.0 & 0.0 & 12000.0 & 100.0 & 100.0 & 108.0 & 85.0 & 85.0 & 2150.6 \\
RicochetRobots & 60.0 & 60.0 & 4942.9 & 25.0 & 25.0 & 9092.0 & 0.0 & 0.0 & 12000.0 & 55.0 & 55.0 & 5491.0 & 40.0 & 40.0 & 7249.0 \\
Sokoban & 65.0 & 65.0 & 4293.2 & 15.0 & 15.0 & 10283.5 & 0.0 & 0.0 & 12000.0 & 45.0 & 45.0 & 6656.0 & 60.0 & 60.0 & 4953.3 \\
Solitaire & 95.0 & 95.0 & 630.4 & 90.0 & 90.0 & 1262.9 & 85.0 & 85.0 & 1919.1 & 95.0 & 95.0 & 618.7 & 95.0 & 95.0 & 603.7 \\
StableMarriage & 100.0 & 100.0 & 486.1 & 0.0 & 0.0 & 12000.0 & 5.0 & 5.0 & 11447.1 & 90.0 & 90.0 & 1510.5 & 65.0 & 65.0 & 4568.0 \\
SteinerTree & 5.0 / 5.0 & 5.0 / 15.0 & 11429.2 / 11411.7 & 5.0 / 5.0 & 5.0 / 5.0 & 11422.7 / 11412.4 & 5.0 / 5.0 & 5.0 / 5.0 & 11404.7 / 11402.5 & 15.0 & 20.0 & 10224.4 & 5.0 & 80.0 & 11400.2 \\
Supertree & 31.7 & 45.0 & 8246.0 & 35.0 & 35.0 & 7854.3 & 11.7 & 11.7 & 10667.8 & 53.3 & 91.7 & 5672.2 & 30.0 & 35.0 & 8481.0 \\
SystemSynthesis & 0.0 / 0.0 & 0.0 / 0.0 & 12000.0 / 12000.0 & 0.0 / 0.0 & 0.0 / 0.0 & 12000.0 / 12000.0 & 35.0 / 90.0 & 40.0 / 95.0 & 8004.6 / 1562.6 & 0.0 & 0.0 & 12000.0 & 0.0 & 0.0 & 12000.0 \\
TravelingSalesPerson & 45.0 / 60.0 & 45.0 / 60.0 & 6764.8 / 4940.8 & 0.0 / 0.0 & 0.0 / 0.0 & 12000.0 / 12000.0 & 90.0 / 95.0 & 90.0 / 100.0 & 1241.9 / 653.0 & 0.0 & 0.0 & 12000.0 & 0.0 & 0.0 & 12000.0 \\
ValvesLocationProblem & 65.0 / 65.0 & 65.0 / 65.0 & 4337.0 / 4267.3 & 70.0 / 75.0 & 70.0 / 75.0 & 3735.6 / 3114.5 & 50.0 / 45.0 & 50.0 / 45.0 & 6064.2 / 6633.7 & 80.0 & 85.0 & 2420.4 & 80.0 & 95.0 & 2451.1 \\
VideoStreaming & 100.0 & 95.0 & 3.3 & 0.0 & 0.0 & 12000.0 & 100.0 & 95.0 & 0.7 & 65.0 & 65.0 & 4223.3 & 0.0 & 0.0 & 12000.0 \\
Visit-all & 100.0 & 100.0 & 55.3 & 15.0 & 15.0 & 10285.3 & 35.0 & 35.0 & 8030.3 & 95.0 & 95.0 & 1101.7 & 40.0 & 40.0 & 7214.9 \\
WeightedSequenceProblem & 100.0 & 100.0 & 33.0 & 100.0 & 100.0 & 15.2 & 100.0 & 100.0 & 2.0 & 100.0 & 100.0 & 10.6 & 100.0 & 100.0 & 60.1 \\
\bottomrule
     \end{tabular}
    
    \label{tab:asp_bench_st_details}
\end{sidewaystable}

\begin{sidewaystable}
\scriptsize
     \centering
     \caption{Detailed Results for ASP benchmark problems (8 threads)}
     \begin{tabular}{lrrrrrrrrrr}
     \toprule
         \textbf{Problem Domain} & \multicolumn{6}{c}{asp-fzn} & \multicolumn{3}{c}{clingo}   \\
         & \multicolumn{3}{c}{CP-SAT strict / non-strict} & \multicolumn{3}{c}{Gurobi strict / non-strict} & &  \\
         &  \multicolumn{1}{c}{\emph{score1}} &  \multicolumn{1}{c}{\emph{score2}} & \multicolumn{1}{c}{\emph{score3}} & \multicolumn{1}{c}{\emph{score1}} & \multicolumn{1}{c}{\emph{score2}} & \multicolumn{1}{c}{\emph{score3}} & \multicolumn{1}{c}{\emph{score1}} & \multicolumn{1}{c}{\emph{score2}} & \multicolumn{1}{c}{\emph{score3}}  \\
         \midrule 
        BayesianNL & 45.0 / 61.7 & 51.7 / 68.3 & 6644.2 / 4666.9 & 60.0 / 70.0 & 68.3 / 78.3 & 4880.8 / 3658.6 & 81.7 & 88.3 & 2263.6 \\
BayesianNL & 45.0 / 61.7 & 51.7 / 68.3 & 6644.2 / 4666.9 & 60.0 / 70.0 & 68.3 / 78.3 & 4880.8 / 3658.6 & 81.7 & 88.3 & 2263.6 \\
BottleFillingProblem & 100.0 & 100.0 & 21.0 & 100.0 & 100.0 & 41.7 & 100.0 & 100.0 & 2.8 \\
CombinedConfiguration & 15.0 / 10.0 & 15.0 / 10.0 & 10215.6 / 10801.3 & 10.0 / 5.0 & 10.0 / 5.0 & 10832.1 / 11456.3 & 55.0 & 55.0 & 5490.8 \\
ConnectedMaximim-densityStillLife & 55.0 / 55.0 & 65.0 / 90.0 & 5508.8 / 5483.4 & 35.0 / 40.0 & 35.0 / 40.0 & 7913.1 / 7267.6 & 60.0 & 65.0 & 4885.5 \\
CrewAllocation & 100.0 & 100.0 & 0.2 & 100.0 & 100.0 & 0.2 & 96.2 & 96.2 & 501.8 \\
CrossingMinimization & 100.0 & 100.0 & 45.6 & 100.0 & 100.0 & 55.8 & 100.0 & 100.0 & 15.4 \\
GracefulGraphs & 60.0 & 60.0 & 4924.8 & 0.0 & 0.0 & 12000.0 & 65.0 & 65.0 & 4265.7 \\
GraphColouring & 70.0 & 70.0 & 3792.2 & 55.0 & 55.0 & 5622.0 & 100.0 & 100.0 & 50.0 \\
HanoiTower & 100.0 & 100.0 & 13.0 & 30.0 & 30.0 & 8461.2 & 100.0 & 100.0 & 2.5 \\
IncrementalScheduling & 65.0 & 65.0 & 4328.9 & 20.0 & 20.0 & 9715.3 & 70.0 & 70.0 & 3631.9 \\
KnightTourWithHoles & 10.0 / 10.0 & 10.0 / 10.0 & 10804.2 / 10803.9 & 45.0 / 40.0 & 45.0 / 40.0 & 6668.1 / 7232.9 & 70.0 & 70.0 & 3610.4 \\
Labyrinth & 50.0 / 60.0 & 50.0 / 60.0 & 6056.7 / 4978.6 & 0.0 / 0.0 & 0.0 / 0.0 & 12000.0 / 12000.0 & 100.0 & 100.0 & 70.3 \\
MarkovNL & 13.3 / 11.7 & 21.7 / 13.3 & 10464.3 / 10651.5 & 8.3 / 30.0 & 8.3 / 30.0 & 11055.3 / 8511.8 & 75.0 & 100.0 & 3103.0 \\
MaxSAT & 100.0 & 100.0 & 58.3 & 90.0 & 90.0 & 1368.6 & 85.0 & 85.0 & 1839.4 \\
MaximalCliqueProblem & 80.0 & 80.0 & 2456.2 & 85.0 & 85.0 & 1862.3 & 50.0 & 65.0 & 6272.6 \\
Nomistery & 60.0 & 60.0 & 4935.7 & 0.0 & 0.0 & 12000.0 & 65.0 & 65.0 & 4297.5 \\
PartnerUnits & 70.0 & 70.0 & 3695.1 & 0.0 & 0.0 & 12000.0 & 75.0 & 75.0 & 3017.1 \\
PermutationPatternMatching & 35.0 & 35.0 & 7836.1 & 25.0 & 25.0 & 9035.4 & 80.0 & 80.0 & 2587.9 \\
QualitativeSpatialReasoning & 25.0 & 25.0 & 9081.9 & 5.0 & 5.0 & 11458.7 & 100.0 & 100.0 & 43.9 \\
RicochetRobots & 95.0 & 95.0 & 782.8 & 0.0 & 0.0 & 12000.0 & 95.0 & 95.0 & 868.8 \\
Sokoban & 65.0 & 65.0 & 4249.0 & 0.0 & 0.0 & 12000.0 & 65.0 & 65.0 & 4309.9 \\
Solitaire & 100.0 & 100.0 & 34.8 & 95.0 & 95.0 & 707.6 & 100.0 & 100.0 & 16.9 \\
StableMarriage & 100.0 & 100.0 & 189.4 & 80.0 & 80.0 & 2951.2 & 100.0 & 100.0 & 49.1 \\
SteinerTree & 5.0 / 5.0 & 5.0 / 10.0 & 11406.1 / 11406.6 & 5.0 / 5.0 & 5.0 / 5.0 & 11404.3 / 11401.5 & 15.0 & 95.0 & 10220.7 \\
Supertree & 41.7 & 48.3 & 7069.7 & 23.3 & 23.3 & 9287.3 & 78.3 & 96.7 & 2735.5 \\
SystemSynthesis & 0.0 / 0.0 & 0.0 / 0.0 & 12000.0 / 12000.0 & 85.0 / 100.0 & 85.0 / 100.0 & 2368.8 / 305.0 & 30.0 & 30.0 & 8593.2 \\
TravelingSalesPerson & 90.0 / 90.0 & 90.0 / 90.0 & 1238.9 / 1261.2 & 95.0 / 95.0 & 95.0 / 100.0 & 618.4 / 610.2 & 0.0 & 0.0 & 12000.0 \\
ValvesLocationProblem & 75.0 / 75.0 & 75.0 / 75.0 & 3091.3 / 3057.6 & 50.0 / 45.0 & 50.0 / 45.0 & 6054.9 / 6630.0 & 90.0 & 100.0 & 1249.6 \\
VideoStreaming & 100.0 & 95.0 & 2.8 & 100.0 & 95.0 & 0.7 & 50.0 & 50.0 & 6000.4 \\
Visit-all & 100.0 & 100.0 & 46.5 & 40.0 & 40.0 & 7297.1 & 100.0 & 100.0 & 31.3 \\
WeightedSequenceProblem & 100.0 & 100.0 & 9.7 & 100.0 & 100.0 & 0.9 & 100.0 & 100.0 & 1.2 \\
\bottomrule
     \end{tabular}
    
    \label{tab:asp_bench_mt_details}
\end{sidewaystable}

\subsection{\solver{} Theory Definition and Command Line Arguments}

\begin{listing}[b]
{
\scriptsize
\begin{Verbatim}
#theory cp {
    var_term  {
    -  : 1, unary
    };
    pos_var_term  {
    };
    sum_term {
    -  : 1, unary;
    *  : 0, binary, left
    };
    dom_term {
    -  : 1, unary;
    .. : 0, binary, left
    };
    dom_term_right  {
    };
    disjoint_term {
    @  : 0, binary, left
    };
    &sum/0 : sum_term, {<=,=,!=,<,>,>=}, var_term, body;
    &minimize/0 : sum_term, directive;
    &dom/0 : dom_term, {=}, pos_var_term, head;
    &disjoint/0 : disjoint_term, head;
    &cumulative/0 : disjoint_term, {<=}, pos_var_term, head;
    &distinct/0 : pos_var_term, head
}.
\end{Verbatim}
}
\vspace*{-0.75\baselineskip}

\caption{Theory specification of \solver{}}
\label{lst:casp_theory}
\end{listing}

\begin{listing}
{\scriptsize
\begin{Verbatim}
> asp-fzn -h
A tool that enables solving ASP programs via FlatZinc solvers.

Usage: asp-fzn [OPTIONS] [INPUT_FILES]...

Arguments:
  [INPUT_FILES]...  Input ASP files to process which are passed on to gringo for grounding. 
                    If no files are provided, ASPIF input is read from stdin

Options:
  -f, --output-fzn <FZN_FILE>          Output file path for the FZN target output file. 
                                       Cannot be used with --solver-id
  -o, --output-ozn <OZN_FILE>          Output file path for the target OZN output file. 
                                       Cannot be used with --solver-id
      --non-strict-ranking             Disable strict ranking in the translation
      --linearize                      Linearize constraints. Always on for MIP solvers 
                                       specified with --solver-id
  -v, --verbose                        Enable verbose output
  -s, --solver-id <SOLVER_ID>          MiniZinc solver ID to use (e.g., "cp-sat", 
                                       "org.chuffed.chuffed", ...) for solving the FlatZinc 
                                       directly. Overrides --fzn-file and --ozn-file options
  -t, --time-limit <SECONDS>           Time limit in seconds for the solving process. 
                                       Only relevant with --solver-id
  -p, --parallel <N_THREADS>           Number of threads to use for parallel solving. 
                                       Only relevant with --solver-id
  -a, --all-solutions                  Compute all solutions instead of just one or whether to
                                       print intermediate solutions for optimization problems. 
                                       Only relevant with --solver-id
      --solution-json                  Output is printed as a JSON stream
      --solver-args <SOLVER_ARGS>      Additional arguments passed on to the FZN solver or 
                                       MiniZinc MIP wrapper. Only relevant with --solver-id
      --gringo-path <GRINGO_PATH>      Path to gringo executable used for grounding 
                                       if not in PATH. Only relevant when input is not ASPIF 
                                       from stdin
      --minizinc-path <MINIZINC_PATH>  Path to MiniZinc installation used for solving 
                                       if not in PATH. Only relevant with --solver-id
  -h, --help                           Print help
  -V, --version                        Print version
\end{Verbatim}
}

\vspace*{-\baselineskip}
\caption{The \solver{} command line tool}\label{lst:aspfzn_cmd}
\end{listing}

\noindent
Listing~\ref{lst:casp_theory} shows the gringo theory specification supported by \solver{} and Listing~\ref{lst:aspfzn_cmd} its command line arguments.

\subsection{CASP Problem Encodings}
\label{app:casp}

The encodings of the problems TLSPS, PMSP, and MAPF  that we used in our experiments are shown in Listings~\ref{lst:tlsps_full}, \ref{lst:pmsp}, and \ref{lst:mapf}, respectively.

\begin{listing}
    {\scriptsize
\begin{Verbatim}
&dom{R..D} = start(J) :- job(J), release(J, R), deadline(J, D).
&dom{R..D} = end(J) :- job(J), release(J, R), deadline(J, D).
&dom{L..H} = duration(J) :- job(J), L = #min{ T : durationInMode(J, _, T) }, 
                            H = #max{ T : durationInMode(J, _, T) }.
1 {modeAssign(J, M) : modeAvailable(J, M)} 1 :- job(J).
:- job(J), modeAssign(J, M), durationInMode(J, M, T), &sum{ duration(J) } != T.
:- job(J), &sum{end(J); -start(J); -duration(J)} != 0.
:- precedence(J,K), &sum{start(J); -end(K)} < 0 .
:- job(J), started(J), &sum{start(J)} != 0 .
1 {workbenchAssign(J, W) : workbenchAvailable(J, W)} 1 :- job(J), workbenchRequired(J).
R {empAssign(J, E) : employeeAvailable(J, E)} R :- job(J), modeAssign(J, M), 
                                                   requiredEmployees(M, R).
R {equipAssign(J, E) : equipmentAvailable(J, E), group(E, G)} R :- job(J), group(_, G), 
                                                                   requiredEquipment(J, G, R).
:- job(J), job(K), linked(J, K), empAssign(J, E), not empAssign(K, E).
&disjoint{ start(J)@duration(J) : workbenchAssign(J,W) } :- workbench(W).
&disjoint{ start(J)@duration(J) : empAssign(J,W) } :- employee(W).
&disjoint{ start(J)@duration(J) : equipAssign(J,W) } :- equipment(W).
start(J,S) :- job(J), &sum{start(J)} = S, S = R..D, deadline(J,D), release(J, R).

#minimize{1,E,J,s2 : job(J), empAssign(J, E), not employeePreferred(J, E)  }. 
#minimize{1,E,P,s3 : project(P), empAssign(J, E), projectAssignment(J, P)}.
&dom{0..H} = delay(J)  :- job(J), horizon(H).
:- job(J), due(J, T), &sum{end(J)} > T, &sum{-1*delay(J); end(J)} != T.
:- job(J), due(J, T), &sum{end(J)} <= T, &sum{delay(J)} != 0.
&minimize{delay(J) : job(J)}.
&dom{0..H} = projectStart(P)  :- project(P), horizon(H).
&dom{0..H} = projectEnd(P)  :- project(P), horizon(H).
&dom{0..H} = completionTime(P)  :- project(P), horizon(H).
1 {firstJob(J) : job(J), projectAssignment(J, P)} 1  :- project(P).
:- firstJob(J), projectAssignment(J, P), &sum{projectStart(P); -start(J)} != 0.
:- job(J), projectAssignment(J, P), &sum{projectStart(P); -start(J)} > 0.
1 {lastJob(J) : job(J), projectAssignment(J, P)} 1  :- project(P).
:- lastJob(J), projectAssignment(J, P), &sum{projectEnd(P); -end(J)} != 0.
:- job(J), projectAssignment(J, P), &sum{projectEnd(P); -end(J)} < 0.
:- project(P), &sum{projectEnd(P); -projectStart(P); -completionTime(P)} != 0.
&minimize{completionTime(P) : project(P)}.
\end{Verbatim}
}

\caption{The TLSPS encoding used by \solver{}}\label{lst:tlsps_full}
\end{listing}


\begin{listing}
    {\scriptsize
\begin{Verbatim}
1 { assigned(J,M) : capable(M,J) } 1 :- job(J).
1 { first(J,M) : capable(M,J) } 1 :- assigned(_,M).
1 { last(J,M) : capable(M,J) } 1 :- assigned(_,M).
:- first(J,M), not assigned(J,M).
:- last(J,M), not assigned(J,M).
1 { next(J1,J2,M) : capable(M,J1), J1 != J2  } 1 :- assigned(J2,M), not first(J2,M).
1 { next(J1,J2,M) : capable(M,J2), J1 != J2  } 1 :- assigned(J1,M), not last(J1,M).
:- next(J1,J2,M), not assigned(J1,M).
:- next(J1,J2,M), not assigned(J2,M).
reach(J1,M) :- first(J1,M).
reach(J2,M) :- reach(J1,M), next(J1,J2,M).
:- assigned(J1,M), not reach(J1,M).

&dom{0..H} = start(J) :- job(J), horizon(H).
&dom{0..H} = compl(J) :- job(J), horizon(H).
&dom{0..H} = makespan :- horizon(H).
processing_time(J2,P) :- next(J1,J2,M), setup(J1,J2,M,S), duration(J2,M,D), P = S+D.
processing_time(J,P) :- first(J,M), duration(J,M,P).
:- job(J), processing_time(J,P), &sum{ compl(J) ; -start(J) } != P. 
:- next(J1,J2,M), &sum{ compl(J1) ; -start(J2) } > 0. 
:- assigned(J,M), release(J,M,T), &sum{ start(J) } < T.
:- job(J1), &sum{ compl(J1); -makespan } > 0.

&minimize{ makespan }.
\end{Verbatim}
}
\caption{The PMSP encoding used by \solver{} and clingcon}\label{lst:pmsp}
\end{listing}


\begin{listing}
    {\scriptsize
\begin{Verbatim}
{ move(A,U,V): edge(U,V) } <= 1 :- agent(A), vertex(V).
{ move(A,U,V): edge(U,V) } <= 1 :- agent(A), vertex(U).
:- move(A,U,_), not start(A,U), not move(A,_,U).
:- move(A,_,U), not goal(A,U), not move(A,U,_).
:- start(A,U), move(A,_,U).
:- goal(A,U), move(A,U,_).
:- start(A,U), not goal(A,U), not move(A,U,_).
:- goal(A,U), not start(A,U), not move(A,_,U).

resolve(A,B,U) :- start(A,U), move(B,_,U), A!=B.
resolve(A,B,U) :- goal(B,U), move(A,_,U), A!=B.
{ resolve(A,B,U); resolve(B,A,U) } >= 1 :- move(A,_,U), move(B,_,U), A<B.
:- resolve(A,B,U), resolve(B,A,U).

&dom{ 0..M } = (A,V) :- agent(A), vertex(V), N = #count{ NA : agent(NA) }, 
                        K = #count{ KV : vertex(KV) }, M=N*K*2.
:- move(A,U,V), &sum{(A,U); -(A,V)} > -1.
:- resolve(A,B,U), move(A,U,V), &sum{(A,V); -(B,U)} > -1.
\end{Verbatim}
}

\caption{The MAPF encoding used by \solver{} and clingcon}\label{lst:mapf}
\end{listing}


\section{Proofs}\label{app:proofs}

\subsection{Proof of Propositions \ref{prop:as_equiv_rank_supp} and \ref{prop:as_equiv_mod_rank_supp}}

\begin{customprop}{\ref{prop:as_equiv_rank_supp}}
For every HCF program $P$, $I\in \AS(P)$ iff $\langle I, \delta \rangle$ is a ranked supported model of $P$ for some level assignment $\delta$.
\end{customprop}
\begin{proof}[Proof (Sketch)]
    For programs without choice and weight rules, this was shown by \cite{DBLP:journals/amai/Ben-EliyahuD94} and adapted for normal programs with weight rules by \cite{DBLP:conf/lpnmr/JanhunenNS09}.

    Adapting the later proof for HCF programs with choice rules is trivial, as the program can be normalized.
\end{proof}

\begin{customprop}{\ref{prop:as_equiv_mod_rank_supp}}
For every HCF program $P$, $I\in \AS(P)$ iff $\langle I, \delta \rangle$ is a modular ranked scc-supported model of $P$ for some level assignment $\delta$.
\end{customprop}
\begin{proof}[Proof (Sketch)]
    This was shown for normal programs without choice rules by \cite{DBLP:conf/lpnmr/JanhunenNS09} and can again easily be adapted for our fragment.
\end{proof}

\subsection{Proof of Theorem \ref{theo:tr-soundness}}

Proving the theorem essentially amounts to showing that each model of $\Tr(P)$ is a modular ranked scc-supported model of $P$, where the core of the argument concerns considering rules, i.e.,  ASP programs. For CASP programs, we in addition have to consider linear variables and linear constraints; however, they carry over directly to $\Tr(P)$ and do not need supportedness, and thus require no special treatment.

\begin{lemma}\label{lem:bd_var_iff_body}
    Let $r$ be a disjunctive rule such that for each $a\in H(r)$, $\SCC_P(a)\cap B^+(r)=\emptyset$ and $\langle I, \delta \rangle\models \mathit{TrRule}(r)$.
    Then, $I\models B(r)$ iff $\langle I, \delta \rangle\models \mathit{bd}_r$. 
\end{lemma}
\begin{proof}
    If $B(r)$ is a normal rule body (\ref{eqn:normal}), then $I\models B(r)$ iff $\langle I, \delta \rangle \models b$ for each $b \in B^+(r)$ and $\langle I, \delta \rangle \models\neg b$ for each $b \in B^-(r)$. 
    Hence, $I\models B(r)$ iff $\langle I, \delta \rangle\models \bigwedge_{b \in B^+(r)} b \bigwedge_{b \in B^-(r)} \neg b$, which by constraint (\ref{eqn:tight_body_normal}) holds iff $\langle I, \delta \rangle\models \mathit{bd}_r$.

    The case when $B(r)$ is a weighted rule body can be shown mutatis mutandis.
\end{proof}

\begin{lemma}\label{lem:bd_var_body_n}
    Suppose $P$ is a HCF program $P$ and $r\in P$ is a disjunctive rule with $H(r)=\{a\}$ or a choice rule, and a normal rule body such that $\SCC_P(a)\cap B^+(r)\neq\emptyset$ and $\langle I, \delta \rangle\models \Tr(P)$.
    Then, $I\models B(r)$ and $\delta(\ell_a) > \mathit{max}_{b\in B^+(r)\cap\SCC_P(a)} \ \delta(\ell_b)$ iff $I\models \bd{r}{a}$. 
\end{lemma}
\begin{proof}
    First note that $\langle I, \delta \rangle\models \Tr(P)$ implies that $\delta(\ell_a) > \delta(\ell_b)$ iff $I \models \mathit{dep}_{a,b}$ for each $b\in \SCC_P(a)$ by constraint (\ref{eqn:dep_var}).
    Hence, $\delta(\ell_a) > \mathit{max}_{b\in B^+(r)\cap\SCC_P(a)} \ \delta(\ell_b)$ iff $I\models \mathit{dep}_{a,b}$ for each $b\in B^+(r)\cap\SCC_P(a)$.
    Furthermore,  $I\models B(r)$ iff $I\models b$ for each $b\in B^+(r)$ and $I\models \neg b$ for each $b\in B^-(r)$.
    Hence, the Lemma follows from the satisfaction of constraint (\ref{eqn:non_tight_body}).
\end{proof}

\begin{lemma}\label{lem:bd_var_body_w}
    Suppose $P$ is a HCF program  and $r\in P$ is a disjunctive rule such that $H(r)=\{a\}$ or a choice rule with weighted rule body, and $\SCC_P(a)\cap B^+(r)\neq\emptyset$, and $\langle I, \delta \rangle\models \Tr(P)$.
    Then, $I\models \bd{r}{a}$ iff 
    \begin{equation}\label{eq:lem5}
    l \ \leq \sum_{b \in (I\cap B^+(r)) \setminus \SCC_P(a)} \mysummanddispq{-0.75cm}{w_b^r} + \sum_{b \in B^+(r) \cap \SCC_P(a), \delta(x_{b}) < \delta(\ell_a)}\mysummanddispq{-0.75cm}{w_b^r} + \sum_{b \in B^-(r)\setminus I}\mysummanddisp{-0.5cm}{w_b^r}.
    \end{equation}
\end{lemma}
\begin{proof}
    Again note that $\langle I, \delta \rangle\models \Tr(P)$ implies that $\delta(\ell_a) > \delta(\ell_b)$ iff $I \models \mathit{dep}_{a,b}$ for each $b\in \SCC_P(a)$ by constraint (\ref{eqn:dep_var}).
    
    ($\Rightarrow$) Suppose $I\models \bd{r}{a}$.
    Since $B(r)$ is a  weighted rule body, $\langle I, \delta \rangle\models \Tr(P)$ implies that constraint (\ref{eqn:body_non_tight_weighted_or}) is satisfied and thus either (i) $\langle I, \delta \rangle\models \mathit{ext}_r^a$ or (ii) $\langle I, \delta \rangle\models \mathit{int}_r^a$.
    If (i) holds, then constraint (\ref{eqn:body_non_tight_weighted_ext}) implies 
    $$l \leq \sum_{b \in B^+(r) \setminus \SCC_P(a)} \mysummanddispq{-0.75cm}{w_i^r} + \sum_{b \in B^-(r) \setminus I} \mysummanddisp{-0.25cm}{w_j^r}$$
    which in turn implies 
    inequality (\ref{eq:lem5}).
    Similarly, if (ii) holds then constraint (\ref{eqn:body_non_tight_weighted_int}) implies 
    inequality (\ref{eq:lem5}).
    
    ($\Rightarrow$) Conversely, suppose that
    %
    inequality (\ref{eq:lem5}) holds. Then constraint (\ref{eqn:body_non_tight_weighted_int}) implies $\langle I, \delta \rangle\models \mathit{int}_r^a$ which in turn implies $\langle I, \delta \rangle\models \bd{r}{a}$ by constraint (\ref{eqn:body_non_tight_weighted_or}) and thus $I\models \bd{r}{a}$.
\end{proof}

\begin{lemma}\label{lem:bd_var_head}
    Suppose $r$ is a disjunctive rule  such that for each $a\in H(r)$, $\SCC_P(a)\cap B^+(r)=\emptyset$, and $\langle I, \delta \rangle\models \mathit{TrRule}(r)$.
    Then, $\langle I, \delta \rangle\models \mathit{bd}_r$ implies $I \models a$ for some $a \in H(r)$. 
\end{lemma}
\begin{proof}
    Towards a contradiction, suppose $I \not\models a$ for every $a \in H(r)$, i.e.,  $I\cap H(r) = \emptyset$.
    If $r$ is a normal rule where $H(r)=\{a\}$, then $\langle I, \delta \rangle\models \mathit{TrRule}(r)$ implies that constraint (\ref{eqn:body_non_tight_supp_equiv}) is satisfied and thus $\langle I, \delta \rangle\models \supp{r}{a}$.
    By constraint (\ref{eqn:normal_supp}), $I \models a$. Contradiction.

    If $|H(r)|>1$, then constraint (\ref{eqn:disj_rule_sat}) and $\langle I, \delta \rangle\models \mathit{bd}_r$ imply $I \models a$ for some $a \in H(r)$.
\end{proof}

\begin{lemma}\label{lem:support_var}
    Suppose $r$ is a partially shifted rule and $I$ an interpretation such that $\langle I, \delta \rangle\models \mathit{TrRule}(r)$. 
    If $\langle I, \delta \rangle \models \supp{r}{a}$ for some $a \in H(r)$, then
    \begin{enumerate}
        \item[(i)] $\langle I, \delta \rangle \models \bd{r}{a}$ whenever $\SCC_P(a)\cap B^+(r)\neq\emptyset$, and
        \item[(ii)] $\langle I, \delta \rangle \models \mathit{bd}_r$ otherwise.
    \end{enumerate}
\end{lemma}
\begin{proof}
    Suppose that $r$ is normal, i.e., $H(r)=\{a\}$. Then the statement follows trivially from constraints (\ref{eqn:body_tight_supp_equiv}) and (\ref{eqn:body_non_tight_supp_equiv}).

    So suppose $|H(r)|>1$, then $\SCC_P(a)\cap B^+(r)=\emptyset$ since $r$ is partially shifted.
    Now, $\langle I, \delta \rangle \models \supp{r}{a}$ and constraint (\ref{eqn:disjunctive_supp}) imply $\langle I, \delta \rangle \models \mathit{bd}_r$.
\end{proof}

\begin{lemma}\label{lem:disj_excl}
    Suppose $r$ is a disjunctive rule $r$ such that for each $a\in H(r)$, $\SCC_P(a)\cap B^+(r)=\emptyset$, and $\langle I, \delta \rangle$ is a ranked interpretation such that $\langle I, \delta \rangle\models \TrHd(r)$. Then, $\langle I, \delta \rangle\models \supp{r}{a}$ for some $a \in H(r)$ implies $I\cap H(r) =\{a\}$.
\end{lemma}
\begin{proof}
    Suppose $\langle I, \delta \rangle\models \supp{r}{a}$ for some $a \in H(r)$. 
    If $H(r)=\{a\}$, then the statement follows directly from constraint (\ref{eqn:normal_supp}), so suppose $|H(r)|>1$.
    Then, by constraint (\ref{eqn:disjunctive_supp}) $\langle I, \delta \rangle\models \supp{r}{a}$ implies $I\cap H(r) \subseteq \{a\}$ and $\langle I, \delta \rangle\models \bd{r}{}$. The latter now implies $I\cap H(r) = \{a\}$ by Lemma~\ref{lem:bd_var_head}. 
\end{proof}

\begin{lemma}\label{lem:trans_rule_mod}
For every rule $r$ of a partially shifted HCF program $P$ and e-interpretation  $\langle I, \delta \rangle$, $\langle I, \delta \rangle \models \Tr(P)$ implies $I\cap \atomsP \models r$.
\end{lemma}
\begin{proof}
    Towards a contradiction, suppose $I \not\models r$.
    Then, $I\models B(r)$ but $I\not\models H(r)$.
    The latter implies that $H(r)$ cannot be a choice head and it is thus a disjunctive head (\ref{eqn:disj}).

    By definition, $\langle I, \delta \rangle \models \Tr(P)$ implies $\langle I, \delta \rangle \models \TrRule(r)$ which in turn implies $\langle I, \delta \rangle \models \TrBd(r)$.
    
    Suppose for each $a\in H(r)$, $\SCC_P(a)\cap B^+(r)=\emptyset$.
    From Lemma~\ref{lem:bd_var_iff_body} and $I\models B(r)$, it follows that $\langle I, \delta \rangle \models \mathit{bd}_r$.
    By Lemma~\ref{lem:bd_var_head}, we thus obtain $I\models a$ for some $a \in H(r)$ and thus $I\models H(r)$, which contradicts the initial assumption that $I\not\models H(r)$.
    

    Suppose $H(r) = \{a\}$ and $\SCC_P(a)\cap B^+(r)\neq\emptyset$. 
    By assumption, $I\not\models a$ and thus $\langle I, \delta \rangle \not\models \supp{r}{a}$ by constraint (\ref{eqn:normal_supp}).
    The latter implies $\langle I, \delta \rangle\not\models \bd{r}{a}$ by (\ref{eqn:body_non_tight_supp_equiv}).
    
    If $B(r)$ is a normal rule body of form (\ref{eqn:normal}), then by Lemma~\ref{lem:bd_var_body_n} we get either 
    (i) $I\not\models B(r)$, or 
    (ii) $\ell_a \leq \mathit{max}_{b\in B^+(r)\cap\SCC_P(a)} \ \ell_b$.
    Case (i) contradicts our assumption that $I\models B(r)$, so assume (ii).
    Given that $a \not\in I$, $\ell_a = |\SCC_P(a)|+1$ by constraint (\ref{eqn:rank_pos_level}) from $\mathit{TrRk(P)}$.
    Furthermore, $I\models B(r)$ implies $B^+(r)\subseteq I$ which implies $\ell_b \leq |\SCC_P(a)|$ for each $b\in \SCC_P(a)$ by $\SCC_P(a)=\SCC_P(b)$ and constraint (\ref{eqn:rank_pos_level}).
    The latter clearly contradicts (ii).

    If $B(r)$ is a weighted rule body (\ref{eqn:weighted}), then by Lemma~\ref{lem:bd_var_body_w}, $\langle I, \delta \rangle\not\models \bd{r}{a}$ implies that inequation (\ref{eq:lem5}) does not hold, and thus we have
    \begin{equation}\label{eq:lem9-1} l \ >\  \sum_{b\in B^+(r) \setminus \SCC_P(a)} \mysummanddispq{-0.75cm}{w_b^r} + \sum_{b\in B^+(r)\cap \SCC_P(a), \delta(x_{b}) < \delta(\ell_a)} \mysummanddispq{-0.95cm}{w_b^r} + \sum_{b\in B^-(r) \setminus I}\mysummanddisp{-0.25cm}{w_b^r.}
    \end{equation}
    From $I\models B(r)$, we obtain
    \begin{equation}\label{eq:lem9-2}
    l\ \leq \sum_{b\in B^+(r) \cap I} w_b^r + \sum_{b\in B^-(r) \setminus I} w_b^r.
    \end{equation}
   From (\ref{eq:lem9-1}) and (\ref{eq:lem9-2}), we obtain
    \begin{align*}
    \sum_{b\in B^+(r)\setminus \SCC_P(a)} \mysummanddispq{-0.75cm}{w_b^r} + \sum_{b\in B^+(r)\cap \SCC_P(a), \delta(\ell_{b}) < \delta(\ell_a)}\mysummanddispq{-0.95cm}{w_b^r} + \sum_{b\in B^-(r) \setminus I} \mysummanddisp{-0.25cm}{w_b^r}\ 
    < \ \sum_{b\in B^+(r)\cap I} \mysummanddisp{-0.25cm}{w_b^r} ~~ + \sum_{b\in B^-(r) \setminus I} \mysummanddisp{-0.25cm}{w_b^r}\ ;
    \end{align*}
   it follows that 
   \begin{equation}\label{eq:lem9-3}
    \sum_{b\in B^+(r) \cap \SCC_P(a), \delta(\ell_{b}) < \delta(\ell_a)}\mysummanddispq{-0.75cm}{w_b^r} < \sum_{b\in B^+(r) \cap  I \cap \SCC_P(a)} \mysummanddisp{-0.75cm}{w_b^r}
    \end{equation}
    holds. Given that $a \not\in I$, we have $\ell_a = |\SCC_P(a)|+1$ by constraint (\ref{eqn:rank_pos_level}) from $\mathit{TrRk(P)}$ and we obtain that 
    $$\{ b\in B^+(r) \mid b \in \SCC_P(a), \ell_{b} < \ell_a\} \subseteq B^+(r) \cap I\cap\SCC_P(a)\ ;$$
    this raises a contradiction with 
    inequation (\ref{eq:lem9-3}).
        
    It remains to consider the case where $r$ is a constraint rule. It 
    can be checked that in this case, $\langle I, \delta \rangle \models \TrRule(r)$ implies $I \not\models B(r)$ by the constraints (\ref{eqn:constraint_normal}) and (\ref{eqn:constraint_weighted}), which is a contradiction.
\end{proof}

\begin{corollary}
    For every partially shifted HCF program $P$ and e-interpretation  $\langle I, \delta \rangle$, $\langle I, \delta \rangle \models \Tr(P)$ implies $I \models P$.
\end{corollary}

\begin{lemma}\label{lem:trans_supp}
    Suppose $P$ is a partially shifted HCF program and $\langle I, \delta \rangle$ is a e-interpretation such that $\langle I, \delta \rangle \models \Tr(P)$. Then for each $a \in I$ there is some rule $r \in P$ which scc-supports $a$ in $\langle I, \delta \rangle$.
\end{lemma}
\begin{proof}
    Let $a\in I$ be arbitrary. 
    First note that $\langle I, \delta \rangle \models \Tr(P)$ implies that constraint (\ref{eqn:support_clause}) is satisfied and thus $\langle I, \delta \rangle \models \supp{r}{a}$ for some rule $r \in P$.
    Suppose $\SCC_P(a)\cap B^+(r)=\emptyset$. Then by Lemma~\ref{lem:support_var}, from $\langle I, \delta \rangle \models \supp{r}{a}$ we obtain $\langle I, \delta \rangle \models \bd{r}{}$.
    If $H(r)$ is a choice (\ref{eqn:choice}), then $r$ supports $a$ in $\langle I, \delta \rangle$.
    So suppose $H(r)$ is a disjunction (\ref{eqn:disj}).
    By Lemma~\ref{lem:disj_excl}, $\langle I, \delta \rangle \models \supp{r}{a}$ implies $I\cap H(r) = \{a\}$ and thus $r$ supports $a$ in $\langle I, \delta \rangle$.

    It remains to consider the case where $\SCC_P(a)\cap B^+(r)\neq\emptyset$. Note that due to our assumptions, $P$ is partially shifted, i.e., $H(r)$ is either a choice head or $H(r) = \{a\}$.
    In any case, $\langle I, \delta \rangle \models \supp{r}{a}$ implies $\langle I, \delta \rangle \models \bd{r}{a}$ by Lemma~\ref{lem:support_var}.
    Suppose $B(r)$ is a normal rule body. Then by Lemma~\ref{lem:bd_var_body_n}, $\langle I, \delta \rangle \models \bd{r}{a}$ implies $I\models B(r)$ and $\delta(\ell_a) > \mathit{max}_{b\in B^+(r)\cap\SCC_P(a)} \ \delta(\ell_b)$.
    Hence, $r$ supports $a$ in $\langle I, \delta \rangle$.

    If $B(r)$ is a weighted rule body (\ref{eqn:weighted}), then by Lemma~\ref{lem:bd_var_body_w}, 
    $$l \leq \sum_{b\in B^+(r)\setminus \SCC_P(a)} \mysummanddispq{-0.75cm}{w_b^r} + \sum_{b\in B^+(r)\cap \SCC_P(a), \delta(\ell_{b}) < \delta(\ell_a)}\mysummanddispq{-0.75cm}{w_b^r} + \sum_{b\in B^-(r) \setminus I} \mysummanddisp{-0.5cm}{w_b^r}$$ holds, which implies that $r$ supports $a$ in $\langle I, \delta \rangle$.
\end{proof}

\begin{customlemma}{\ref{lem:trans_mod_rank_supp}}
    For every partially shifted HCF program $P$,  if $\langle I , \delta \rangle \models \Tr(P)$ then $\langle I\cap \atomsP, \delta' \rangle$ is a modular ranked scc-supported model of $P$, where $\delta'(\ell_a)=1$ for $a\in I$ s.t. $|\SCC_P(a)|=1$ and $\delta'(\ell_a)=\infty$ for $a \in \atomsP \setminus I$.
\end{customlemma}
\begin{proof}
    From Lemma \ref{lem:trans_rule_mod}, we get that $I\cap \atomsP$ is a model of $P$ and from Lemma~\ref{lem:trans_supp}, every rule is scc-supported in $\langle I, \delta \rangle$ and thus $\langle I\cap \atomsP, \delta' \rangle$.
    Hence, $\langle I\cap \atomsP, \delta' \rangle$ is a modular ranked scc-supported model of $P$.
\end{proof}

\begin{customthm}{\ref{theo:tr-soundness}}
 For every partially shifted HCF program $P$, if  $\langle I, \delta \rangle \models \Tr(P)$ then  $\langle I', \delta' \rangle\in \AS(P)$, where $I' = I \cap \atomsP$ and $\delta'(v)=\delta(v)$ for each $v\in \varsP$.
\end{customthm}
\begin{proof}
    For plain ASP programs where $\varsP = \emptyset$, this follows from Lemma~\ref{lem:trans_mod_rank_supp} and Proposition~\ref{prop:as_equiv_rank_supp}.
    For proper CASP programs, the additional linear constraints were considered to be in $\Tr(P)$ and are thus satisfied. Furthermore, given that every $a\in \linAtomsP$, is considered to be classical, i.e., does not require support, we have that $I$ is an answer set and $\langle I', \delta' \rangle$ is a constraint answer set.
\end{proof}

\subsection{Proof of Theorem \ref{theo:tr-completeness}}

\begin{definition}
    Given a modular ranked supported model $\langle I, \delta \rangle$ of a program $P$, we say that an atom $a \in I$ is \emph{externally supported} if there is some rule $r\in P$ which scc-supports $a$ and 
    (i) $B^+(r)\cap\SCC_P(a) = \emptyset$, if $r$ has a normal rule body, or
    (ii) 
    \begin{equation}\label{eq:defn3}
    l_r\ \leq \sum_{b\in B^+(r)\setminus \SCC_P(a)} \mysummanddispq{-0.75cm}{w_b^r} + \sum_{b\in B^-(r)\setminus I} \mysummanddisp{-0.25cm}{w_b^r},
    \end{equation}
    if $r$ has a weighted rule body of form (\ref{eqn:weighted}).
\end{definition}
\begin{definition}\label{def:strict}
    A modular ranked supported model $\langle I, \delta \rangle$ of a program $P$ is called \emph{strict} if for each for each $a \in I$ with $|\SCC_P(a)|>1$, it holds that
    $$\delta(\ell_a) = \left\{\begin{array}{l} 1,\qquad\qquad\quad~~ \text{if some $r\in P$ with $B^+(r)\cap\SCC_P(a) = \emptyset$ externally supports $a$,} \\
        \mathit{min} \{ \ \mathit{max}\{ \delta(\ell_b) \mid b \in B^+(r)\cap\SCC_P(a) \} \  \mid r \in P, r \text{ scc-supports } a \} \  \text{ otherwise}.
    \end{array}\right.$$
\end{definition}

Note that we will also call models of $\Tr(P)$ strict whenever the described property holds for each $\ell_a$. 

\begin{proposition}\label{prop:as_ranked}
    For every HCF program $P$ and  $I\in \AS(P)$,  there exists some modular ranked scc-supported model $\langle I, \delta \rangle$ of $P$ which is strict.
\end{proposition}
\begin{proof}[Proof (Sketch)]
    By Proposition~\ref{prop:as_equiv_mod_rank_supp},  there exists some modular ranked scc-supported model $\langle I, \delta' \rangle$ of $P$.
    It is not hard to see that $\langle I, \delta \rangle$ can be obtained from $\langle I, \delta' \rangle$ by removing gaps from the level mapping to achieve strictness.
\end{proof}

\begin{customthm}{\ref{theo:tr-completeness}}
 For every partially shifted HCF program $P$ and answer set $\langle I , \delta \rangle$ of $P$, there exists some e-interpretation  $\mathcal{I'}=\langle I', \delta' \rangle$ s.t. $I'\cap \atomsP=I\cap \atomsP$, $\delta'(v)=\delta(v)$ for $v\in \varsP$, and $\mathcal{I}' \models \Tr(P)$.
\end{customthm}   
\begin{proof}
Let $I\in \AS(P)$.
Then by Proposition~\ref{prop:as_ranked}, there is some modular ranked supported model $\langle I, \delta \rangle$ of $P$ s.t.
for each $a \in I$ where  
$|\SCC_P(a)|>1$ it holds that
$\delta(\ell_a) = \mathit{max}(1,  \mathit{min} \{ \ \mathit{max}\{ \delta(\ell_b) \mid b \in B^+(r)\cap\SCC_P(a) \} \  \mid r \in P,\ H(r)\cap I = \{ a \}, \ I\models B(r) \})$.

We will construct $I'$ from $I$ as follows.
For every $a,b \in I$, whenever $\delta(\ell_a) > \delta(\ell_b)$ then $\mathit{dep}_{a,b}$ is considered to be in $I'$.
Furthermore, if $\delta(\ell_a) > \delta(\ell_b) + 1$, then $\mathit{y}_{a,b}$ and $\mathit{gap}_{a,b}$ are in $I'$.
From this we can see that $\langle I', \delta \rangle \models \TrRk(P)$.

1) Now, for each locally tight rule $r\in P$ s.t. $I\models B(r)$, we consider $\bd{r}{}$ to be in $I'$.
Furthermore, if $r$ is a disjunctive rule and $H(r)\cap I = \{a\}$ for some $a \in H(r)$, we also consider $\supp{r}{a}$ to be in $I'$. 
Similarly, whenever $r$ is a choice rule, $\supp{r}{a}$ is considered to be in $I'$ for each $a\in H(r)\cap I$.

We claim that $\langle I', \delta \rangle \models \TrRule(r)$.
From $I\models B(r)$ and $\bd{r}{} \in I'$ it is not hard to 
check that $\langle I', \delta \rangle \models \TrBd(r)$ holds in this case.
If $|H(r)|=1$ or $r$ is a choice rule, then $\supp{r}{a}$ is in $I'$ for each $a\in H(r)\cap I$. Since $\bd{r}{}$ is also in $I'$, it holds that $\langle I', \delta \rangle \models \TrRule(r)$.

If $r$ is a disjunctive rule and $|H(r)| > 1$, we consider two cases. 
First, assume that $H(r)\cap I = \{a\}$ for some $a \in H(r)$. By construction, $\supp{r}{a}$ and $\bd{r}{}$ are in $I'$ and $H(r)\cap I = \{a\}$ implies $\langle I', \delta \rangle \models \TrRule(r)$.
Otherwise, $|H(r)\cap I| > 1$  and thus both sides of the equivalence in constraint (\ref{eqn:disjunctive_supp}) evaluate to false, thus again $\langle I', \delta \rangle \models \TrRule(r)$.

2) Suppose that $r$ is not locally tight. Since $P$ is partially shifted, $r$ is either a choice rule or a normal rule.
If there is some $a \in H(r)$ s.t.\ $\SCC_P(a)\cap B^+(r)=\emptyset$, we again consider $\bd{r}{}$ to be in $I'$ whenever $I\models B(r)$.
Now, for each $a \in H(r)$ s.t.\ $\SCC_P(a)\cap B^+(r)=\emptyset$, the translation contains the same constraints as above which again are satisfied.
Furthermore, since $\ell_a=1$ by assertion about $\langle I,\delta\rangle$, constraint (\ref{eqn:normal_ext_rk_one}) is satisfied as well.

It remains to consider $a \in H(r)$ s.t.\ $\SCC_P(a)\cap B^+(r)\neq\emptyset$.
Consider first the case when $B(r)$ is a normal rule body.
For each $a \in H(r)$ s.t.\ $\SCC_P(a)\cap B^+(r)\neq\emptyset$ and $\ell_a = 1 + \mathit{max}\{\ell_b \mid b \in B^+(r) \}$, we consider $\bd{r}{a}\in I'$.
Note that by this construction, $\langle I', \delta \rangle$ satisfies the constraints (\ref{eqn:non_tight_body}) and (\ref{eqn:non_tight_normal_gap}).
    
Consider then that $B(r)$ is a weighted rule body. Then for each $a \in H(r)$, we add $\mathit{ext}_r^a$ and/or $\mathit{int}_r^a$ to $\langle I', \delta \rangle$ depending on whether there is external and/or internal support. Similarly,  $\bd{r}{a}$ is included in $I'$ whenever there is some support.
Hence, constraints (\ref{eqn:body_non_tight_weighted_ext}), (\ref{eqn:body_non_tight_weighted_int}) and (\ref{eqn:body_non_tight_weighted_or}) are satisfied by construction.

Note that by assumption, for each $b \in \SCC_P(a)\cap B^+(r)$, we have that $\langle I', \delta \rangle \not\models \mathit{gap}_{a,b}$ since there are no gaps in the level mapping. Informally, $\mathit{int}_r^a$ thus implies $\mathit{aux}_r^a$ and constraint (\ref{eqn:non_tight_weighted_gap_aux}) can be satisfied by construction if we add $\mathit{aux}_r^a$ whenever $\mathit{int}_r^a$ has been added.
Constraint (\ref{eqn:weighted_ext_rk_one}) is further satisfied, since we assume that externally supported atoms have rank 1 and the left-hand side is thus less or equal to $2\cdot s_a + 1$. 
Furthermore, $\supp{r}{a}$ is in $I'$ whenever $\bd{r}{a}$ is.

It remains to show that constraint (\ref{eqn:support_clause}) is satisfied for each atom $a \in \atomsP\setminus\linAtomsP$.
Given that $I\in AS(P)$, there is some scc-supporting rule $r$ for $a$.
Furthermore, we already have established above that $\langle I', \delta \rangle \models \TrRule(r)$.
Hence, by Lemmas \ref{lem:bd_var_body_n} and \ref{lem:bd_var_body_w}, either $\langle I', \delta \rangle \models \bd{r}{}$ or $\langle I', \delta \rangle \models \bd{r}{a}$.
In either case, we obtain $\langle I', \delta \rangle \models \supp{r}{a}$, from constraints (\ref{eqn:body_non_tight_supp_equiv}), (\ref{eqn:body_tight_supp_equiv}) and (\ref{eqn:disjunctive_supp}). 
This implies that $\langle I', \delta \rangle$ satisfies the support clause for every $a$ and the constraint (\ref{eqn:support_clause}) is satisfied.
\end{proof}
    
\subsection{Proof of Theorem \ref{theo:tr-1-1-correspondence}}



\begin{lemma}\label{lem:delta_deter_mod}
    Suppose $P$ is a partially shifted HCF program and $\langle I, \delta \rangle$, $\langle I', \delta' \rangle$ are models of $\Tr(P)$, i.e., 
    $\langle I, \delta \rangle \models \Tr(P)$ and $\langle I', \delta' \rangle \models \Tr(P)$, 
    such that $I\cap \atomsP = I'\cap \atomsP$. 
    If $\delta = \delta'$, then $I=I'$.
\end{lemma}
\begin{proof}
    We need to show that the auxiliary Boolean atoms introduced by the translation match for both models.

    For the auxiliary atoms occurring in $\TrRk(P)$, this is clearly the case as they are defined, by the constraints of $\TrRk(P)$, through the rank variables $\ell_a$.

    For the auxiliary body atoms, the equivalence follows from Lemmas \ref{lem:bd_var_body_n} and \ref{lem:bd_var_body_w}.
    
    Support atoms $\supp{r}{a}$ are determined through constraints (\ref{eqn:disjunctive_supp}), (\ref{eqn:normal_supp}),  (\ref{eqn:body_tight_supp_equiv}),(\ref{eqn:body_non_tight_supp_equiv}) and the respective body or head atoms and thus cannot differ in $I$ and $I'$

    Lastly, potential auxiliary atoms $\mathit{ext}_r^a$, $\mathit{int}_r^a$, and $\mathit{aux}_r^a$ are defined by their respective constraints (\ref{eqn:body_non_tight_weighted_ext}), (\ref{eqn:body_non_tight_weighted_int}), and (\ref{eqn:non_tight_weighted_gap}) in $\Tr(P)$ and are linked to other atoms which match in $I$ and $I'$.
\end{proof}

\begin{customlemma}{\ref{lem:trans_strict}}
Suppose  $P$ is a partially shifted HCF program and $\mathcal{I} = \langle I, \delta \rangle$, $\mathcal{I}' = \langle I', \delta' \rangle$ 
    are models of $\Tr(P)$. Then 
$I\cap \atomsP=I'\cap \atomsP$ implies $\delta(\ell_a) = \delta'(\ell_a)$ for every $a \in \atomsP$.
\end{customlemma}
\begin{proof}
    We claim that $\langle I', \delta \rangle \models \Tr(P)$ implies that the ranking defined by $\delta$ and $\delta'$ over the rank variables $\ell_a$ is strict as by Definition~\ref{def:strict}.
    The required rank~1 for externally supported atoms is enforced by constraints (\ref{eqn:normal_ext_rk_one}) and (\ref{eqn:weighted_ext_rk_one}). atom For atoms supported internally, constraint (\ref{eqn:non_tight_normal_gap}) ensures that there can be no gap between the ranks for support from a normal body. 
    while for support 
    from weighted bodies this is enforced through constraints (\ref{eqn:non_tight_weighted_gap}) and (\ref{eqn:non_tight_weighted_gap_aux}), where the latter expresses that internal support is either accompanied by external support as well or that the no gap constraint (\ref{eqn:non_tight_weighted_gap}) holds.

    Given that these constraints are satisfied, the rankings defined by $\delta$ and $\delta'$ are both strict and thus $\delta(\ell_a) = \delta'(\ell_a)$ for every $a \in \atomsP$.
\end{proof}    

\begin{customthm}{\ref{theo:tr-1-1-correspondence}}
For every partially shifted HCF program $P$, there exists a 1-1 mapping between $\AS(P)$ and the models of $\Tr(P)$.
\end{customthm}   
\begin{proof}[Proof (Sketch)]
    Note that Theorem~\ref{theo:tr-soundness} already establishes that every model of $\Tr(P)$ maps to exactly one answer set.
    For the other direction, consider $I\in AS(P)$. 
    From Theorem~\ref{theo:tr-completeness}, we have that there exists some model $\langle I', \delta \rangle$ of $\Tr(P)$ such that $I'\cap \atomsP = I$.
    Hence, we only need to show that for each model $\langle I'', \delta' \rangle$ of $\Tr(P)$ such that $I''\cap \atomsP = I$ and $\delta'(v)=\delta(v)$ for $v\in \varsP$, it holds that $I''=I'$ and $\delta(\ell_a) = \delta'(\ell_a)$ for every $a \in \atomsP$.

    By Lemma~\ref{lem:trans_strict}, $\delta(\ell_a) = \delta'(\ell_a)$ for every $a \in \atomsP$ indeed holds. As   also $I'\cap \atomsP = I''\cap \atomsP = I$ holds, by Lemma~\ref{lem:delta_deter_mod} it follows that  $I'=I''$.
    Hence, for each answer set $I$ of $P$ there is exactly one corresponding model $\langle I', \delta \rangle$ of $\Tr(P)$ such that $I'\cap \atomsP = I$.
\end{proof}

\end{document}